\documentclass{article}

\usepackage[left=1in,right=1in,top=1in,bottom=1in]{geometry}
\RequirePackage{amsthm,amsmath,amsmath,amsfonts,amssymb}

\usepackage{mathabx}

\allowdisplaybreaks

\usepackage{algorithm}
\usepackage[noend]{algpseudocode}

\usepackage{color}

\RequirePackage{amsthm,amsmath,amsmath,amsfonts,amssymb}

\usepackage{mathabx}

\usepackage{algorithm}
\usepackage[noend]{algpseudocode}

\usepackage{color}

\usepackage{bm}
\usepackage{hyperref}
\hypersetup{colorlinks,
	linkcolor=blue,
	citecolor=blue,
	urlcolor=magenta,
	linktocpage,
	plainpages=false}

\usepackage{enumerate}

\usepackage{natbib}
\usepackage{amsmath}
\usepackage{amsthm}
\usepackage{amssymb}
\usepackage{mathabx}
\usepackage{tikz}
\usepackage{xcolor}
\usetikzlibrary{arrows}

\usepackage[title]{appendix}

\usepackage{hyperref}
\usepackage{bm}

\allowdisplaybreaks[4]

\usepackage{caption}
\usepackage{subcaption}

\newcommand{\argmin}{\mathrm{argmin}}

\newcommand{\poly}{\mathrm{poly}}

\renewcommand{\d}{\mathrm{d}}
\newcommand{\dt}{\mathrm{d}t}

\def\R{\mathbb{R}}
\def\N{\mathbb{N}}

\def\cA{\mathcal{A}}
\def\cB{\mathcal{B}}
\def\cC{\mathcal{C}}
\def\cD{\mathcal{D}}
\def\cE{\mathcal{E}}
\def\cF{\mathcal{F}}
\def\cG{\mathcal{G}}
\def\cH{\mathcal{H}}
\def\cI{\mathcal{I}}
\def\cJ{\mathcal{J}}
\def\cK{\mathcal{K}}
\def\cL{\mathcal{L}}
\def\cM{\mathcal{M}}
\def\cN{\mathcal{N}}

\def\cS{\mathcal{S}}

\newcommand{\mat}[1]{\bm{#1}}
\newcommand{\vect}[1]{\bm{#1}}
\newcommand{\norm}[1]{\left\|#1\right\|}
\newcommand{\abs}[1]{\left|#1\right|}

\newtheorem{thm}{Theorem}[section]
\newtheorem{lem}{Lemma}[section]
\newtheorem{cor}{Corollary}[section]

\newtheorem{claim}{Claim}[section]

\title{Algorithmic Regularization in Learning Deep Homogeneous Models: Layers are Automatically Balanced}

\date{}
\author{
  Simon S. Du\thanks{Machine Learning Department, School of Computer Science, Carnegie
  	Mellon University. Email: \texttt{ssdu@cs.cmu.edu}}
  \and
  Wei Hu\thanks{Computer Science Department, Princeton University. Email: \texttt{huwei@cs.princeton.edu}}
  \and
  Jason D. Lee\thanks{Department of Data Sciences and Operations, Marshall School of
  	Business, University of Southern California. Email: \texttt{jasonlee@marshall.usc.edu}}
}

\begin{document}

\maketitle

\begin{abstract}

We study the implicit regularization imposed by gradient descent for learning multi-layer homogeneous functions including feed-forward fully connected and convolutional deep neural networks with linear, ReLU or Leaky ReLU activation.
We rigorously prove that gradient flow (i.e. gradient descent with infinitesimal step size) effectively enforces the differences between squared norms across different layers to remain \emph{invariant} without any explicit regularization.
This result implies that if the weights are initially small, gradient flow automatically balances the magnitudes of all layers.
Using a discretization argument, we analyze gradient descent with positive step size for the non-convex low-rank asymmetric matrix factorization problem without any regularization.
Inspired by our findings for gradient flow, we prove that gradient descent with step sizes $\eta_t = O\left(t^{-\left( \frac12+\delta\right)} \right)$ ($0<\delta\le\frac12$) automatically balances two low-rank factors and converges to a bounded global optimum. Furthermore, for rank-$1$ asymmetric matrix factorization we give a finer analysis showing gradient descent with constant step size converges to the global minimum at a globally linear rate.
We believe that the idea of examining the invariance imposed by first order algorithms in learning homogeneous models could serve as a fundamental building block for studying optimization for learning deep models.

\end{abstract}

\section{Introduction}
\label{sec:intro}

Modern machine learning models often consist of multiple layers.
%introducing two examples
For example, consider a feed-forward deep neural network that defines a prediction function
\begin{align*}
\vect{x} \mapsto f(\vect{x}; \mat{W}^{(1)},\ldots,\mat{W}^{(N)}) = \mat{W}^{(N)}\phi(\mat{W}^{(N-1)}\cdots\mat{W}^{(2)}\phi(\mat{W}^{(1)}\vect{x})\cdots),
\end{align*} 
where $\mat{W}^{(1)},\ldots,\mat{W}^{(N)}$ are weight matrices in $N$ layers, and $\phi\left(\cdot\right)$ is a point-wise \emph{homogeneous} activation function such as Rectified Linear Unit (ReLU) $\phi(x) = \max\{x, 0\}$.
A simple observation is that this model is \emph{homogeneous}: if we multiply a layer by a positive scalar $c$ and divide another layer by $c$, the prediction function remains the same, e.g. $f(\vect{x}; c\mat{W}^{(1)},\ldots,\frac{1}{c}\mat{W}^{(N)}) = f(\vect{x}; \mat{W}^{(1)},\ldots,\mat{W}^{(N)})$.
%This property holds for many commonly used activation functions like ReLU and Leaky ReLU.

A direct consequence of homogeneity is that a solution can produce small function value while being unbounded, because one can always multiply one layer by a huge number and divide another layer by that number.
Theoretically, this possible unbalancedness poses significant difficulty in analyzing first order optimization methods like gradient descent/stochastic gradient descent (GD/SGD), because when parameters are not a priori constrained to a compact set via either coerciveness\footnote{A function $f$ is coercive if $\norm{\vect x} \to \infty$ implies $f(\vect x) \to \infty$.} of the loss or an explicit constraint, GD and SGD are not even guaranteed to converge \citep[Proposition 4.11]{lee2016gradient}. In the context of deep learning, \cite{shamir2018resnets} determined that the primary barrier to providing algorithmic results is in that the sequence of parameter iterates is possibly unbounded.

Now we take a closer look at asymmetric matrix factorization, which is a simple two-layer homogeneous model.
%is asymmetric matrix factorization, consider a prediction function $f(\cdot,\cdot)$ which takes two vectors $\mat{U} \in \mathbb{R}^{d_1 \times k}, \mat{V} \in \mathbb{R}^{d_2 \times r}$ as input and outputs their product, i.e., 
%$g(\mat{U},\mat{V}) = \mat{U}\mat{V}^\top$.
%For any given pair $\left(\vect{U},\vect{V}\right)$, the pair $\left(c\vect{U},\frac{1}{c}\vect{V}\right)$ with $c \neq 0$ is equivalent to $\left(\vect{U},\vect{V}\right)$, because they give the same prediction function: $f(c\vect{U},\frac{1}{c}\vect{V}) = \vect{U}\vect{V}^\top = f\left(\vect{U},\vect{V}\right)$. 
%This model is clearly homogeneous since we have $g(c\vect{U},\frac{1}{c}\vect{V}) = \vect{U}\vect{V}^\top = g\left(\vect{U},\vect{V}\right)$ for any scalar $c\not=0$.
Consider the following formulation for factorizing a low-rank matrix:
\begin{align}
	\min_{\mat{U} \in \mathbb{R}^{d_1 \times r},\mat{V} \in \mathbb{R}^{d_2 \times r}} f\left(\mat{U},\mat{V}\right) = \frac{1}{2}\norm{\mat{U}\mat{V}^\top -\mat{M}^*}_F^2, \label{eqn:intro_mf_obj}
\end{align} where $\mat{M}^* \in \mathbb{R}^{d_1 \times d_2}$ is a matrix we want to factorize.
We observe that due to the homogeneity of $f$, it is not smooth\footnote{A function is said to be smooth if its gradient is $\beta$-Lipschitz continuous for some finite $\beta>0$.} even in the neighborhood of a globally optimum point.
To see this, we compute the gradient of $f$:
\begin{align}
	\frac{\partial f\left(\mat{U},\mat{V}\right)}{\partial \mat{U}} = \left(\mat{U}\mat{V}^\top - \mat{M}^*\right)\mat{V}, \qquad 	\frac{\partial f\left(\mat{U},\mat{V}\right)}{\partial \mat{V}} = \left(\mat{U}\mat{V}^\top - \mat{M}^*\right)^\top\mat{U}. \label{eqn:mf_gradient}
\end{align}
Notice that the gradient of $f$ is not homogeneous anymore.
Further, consider a globally optimal solution $(\mat U, \mat V)$ such that $\norm{\mat{U}}_F$ is of order $\epsilon$ and $\norm{\mat{V}}_F$ is of order $1/\epsilon$ ($\epsilon$ being very small).
%the case when $\norm{\mat{U}\mat{V}}_F$ is $O(1)$, $\norm{\mat{U}}_F$ is of order $\epsilon$ and $\norm{\mat{V}}_F$ is of order $1/\epsilon$. 
A small perturbation on $\mat{U}$ can lead to dramatic change to the gradient of $\mat{U}$.
This phenomenon can happen for all homogeneous functions when the layers are unbalanced.
The lack of nice geometric properties of homogeneous functions due to unbalancedness makes first-order optimization methods difficult to analyze.

%A direct consequence is that when one tries to learn a multi-layer model by solving a (non-convex) program, the solution is not unique.
%In matrix factorization problem, given a rank-$1$ matrix $\mat{M} \in \mathbb{R}^{d_1 \times d_2}$, the goal is to find a pair $\vect{x},\vect{y}$ to solve the following problem \begin{align}
%\min_{\vect{x}\in \mathbb{R}^{d_1},\vect{y} \in\mathbb{R}^{d_2}}\norm{f(\vect{x},\vect{y})-\mat{M}}_F^2 \label{eqn:rank_1_mf_obj}
%\end{align} is minimized.
%Suppose $\mat{M}$ admits \simon{better words?} the singular value decomposition $\mat{M} = \sigma_1\vect{u}\vect{v}^\top$ with $\sigma_1 > 0$ the singular value and $\vect{u} \in \mathbb{R}^{d_1}, \vect{v} \in \mathbb{R}^{d_2}$ the orthonormal left and right singular vectors, respectively.
%The any pair $\left(\vect{x},\vect{y}\right)$ of the form $\left(c \sqrt{\sigma_1}\vect{u},\frac{1}{c}\sqrt{\sigma_1}\vect{v}\right)$ with any $c \neq 0$ is a global optimum solution to Program~\eqref{eqn:rank_1_mf_obj}.

A common theoretical workaround is to artificially  modify the natural objective function as in \eqref{eqn:intro_mf_obj} in order to prove convergence.
In~\citep{tu2015low, ge2017no}, a regularization term for balancing the two layers is added to \eqref{eqn:intro_mf_obj}:\begin{align}
\min_{\vect{U}\in\mathbb{R}^{d_1 \times r},\vect{V}\in\mathbb{R}^{d_2 \times r}} \frac12 \norm{\mat{U}\mat{V}^\top-\mat{M}}_F^2 + \frac{1}{8}\norm{\mat{U}^\top\mat{U}-\mat{V}^\top\mat{V}}_F^2. \label{eqn:intro_mf_reg_obj}
\end{align}
For problem~\eqref{eqn:intro_mf_reg_obj}, the regularizer removes the homogeneity issue and the optimal solution becomes unique (up to rotation).
\citet{ge2017no} showed that the modified objective \eqref{eqn:intro_mf_reg_obj} satisfies (i) every local minimum is a global minimum, (ii) all saddle points are strict\footnote{A saddle point of a function $f$ is strict if the Hessian at that point has a negative eigenvalue.}, and (iii) the objective is smooth. These imply that (noisy) GD finds a global minimum \citep{ge2015escaping, lee2016gradient, panageas2016gradient}.

%Similarly, \citet{du2017convolutional} considered the problem of learning a non-overlapping  one-hidden-layer convolutional neural network with Gaussian input, which reduces to the following non-convex program:
%\begin{align}L\left(\vect{w},\vect{a}\right) = \expect_{\vect{Z}_1,\ldots,\vect{Z}_k \sim N(\vect{0},\mat{I})}\left(\sum_{i=1}^{k}a_i\phi{\vect{w}^\top \vect{Z}_i} - \sum_{i=1}^{k}a^*[i]\phi{\vect{w}_*^\top \vect{Z}_i}\right). \label{eqn:cnn_without}
%\end{align} where $\phi$ is ReLU activation function.
%Note in this formulation, the prediction $f(\vect{w},\vect{a},\vect{Z}_1,\ldots,\vect{Z}_k) = \sum_{i=1}^{k}a_i\phi{\vect{w}^\top \vect{Z}_i}$ is a homogeneous function. 
%\citet{du2017spurious} used \emph{weight normalization} to remove the homogeneity, which reduces to the following optimization program:
%\begin{align}L\left(\vect{w},\vect{a}\right) = \expect_{\vect{Z}_1,\ldots,\vect{Z}_k \sim N(\vect{0},\mat{I})}\left(\sum_{i=1}^{k}a_i\frac{\relu{\vect{v}^\top \vect{Z}_i}}{\norm{\vect{v}}_2} - \sum_{i=1}^{k}a^*[i]\relu{\vect{w}_*^\top \vect{Z}_i}\right) \label{eqn:cnn_with}
%\end{align}
%and they showed GDt for Problem~\eqref{eqn:cnn_with} converges to the global minimum with random initialization.

\begin{figure*}[t!]
	\centering
	\begin{subfigure}[t]{0.45\textwidth}
		\includegraphics[width=\textwidth]{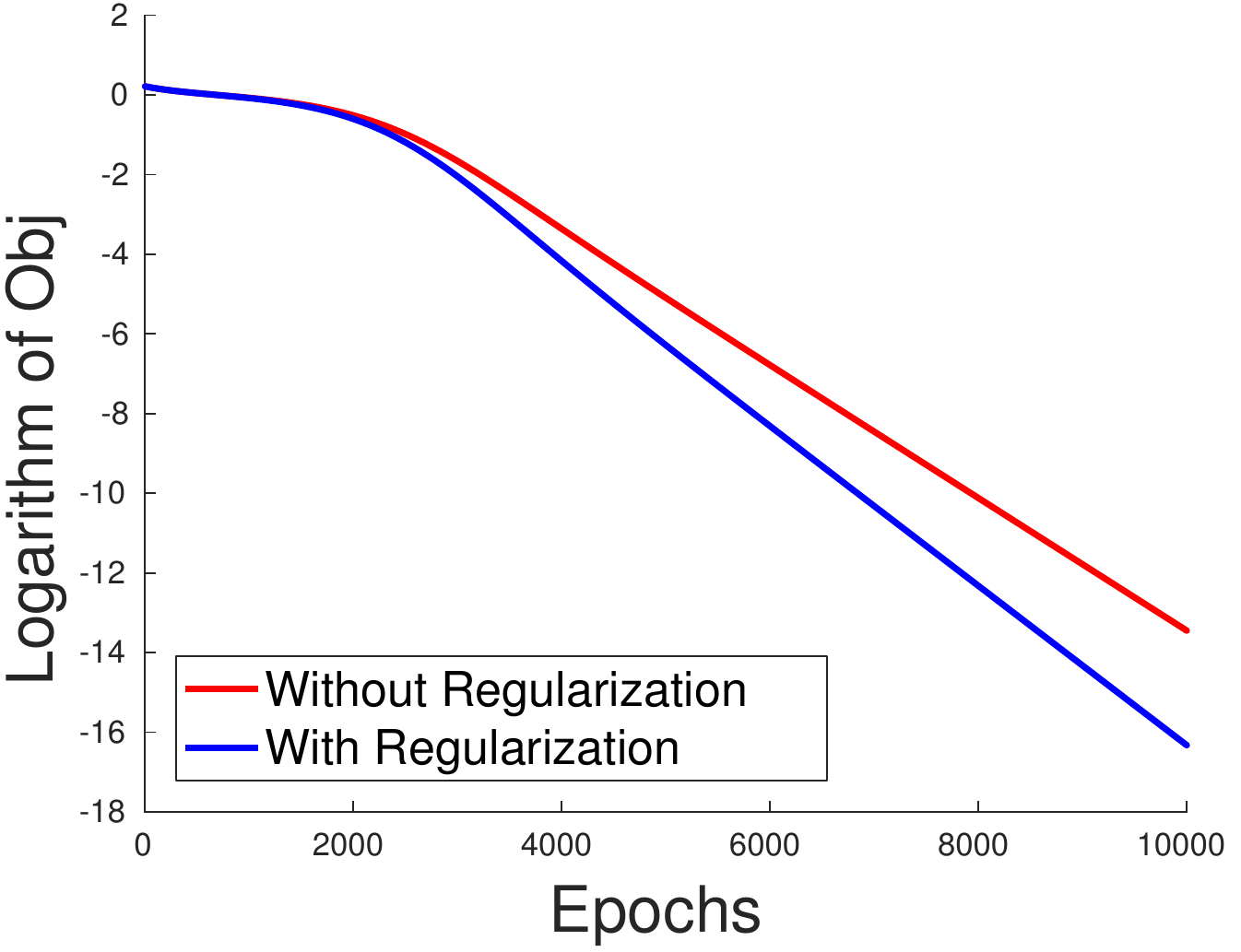}
		\caption{Comparison of convergence rates of GD for objective functions~\eqref{eqn:intro_mf_obj} and~\eqref{eqn:intro_mf_reg_obj}.
			}
		\label{fig:mf_converge}
	\end{subfigure}	
	\quad
	\begin{subfigure}[t]{0.45\textwidth}
		\includegraphics[width=\textwidth]{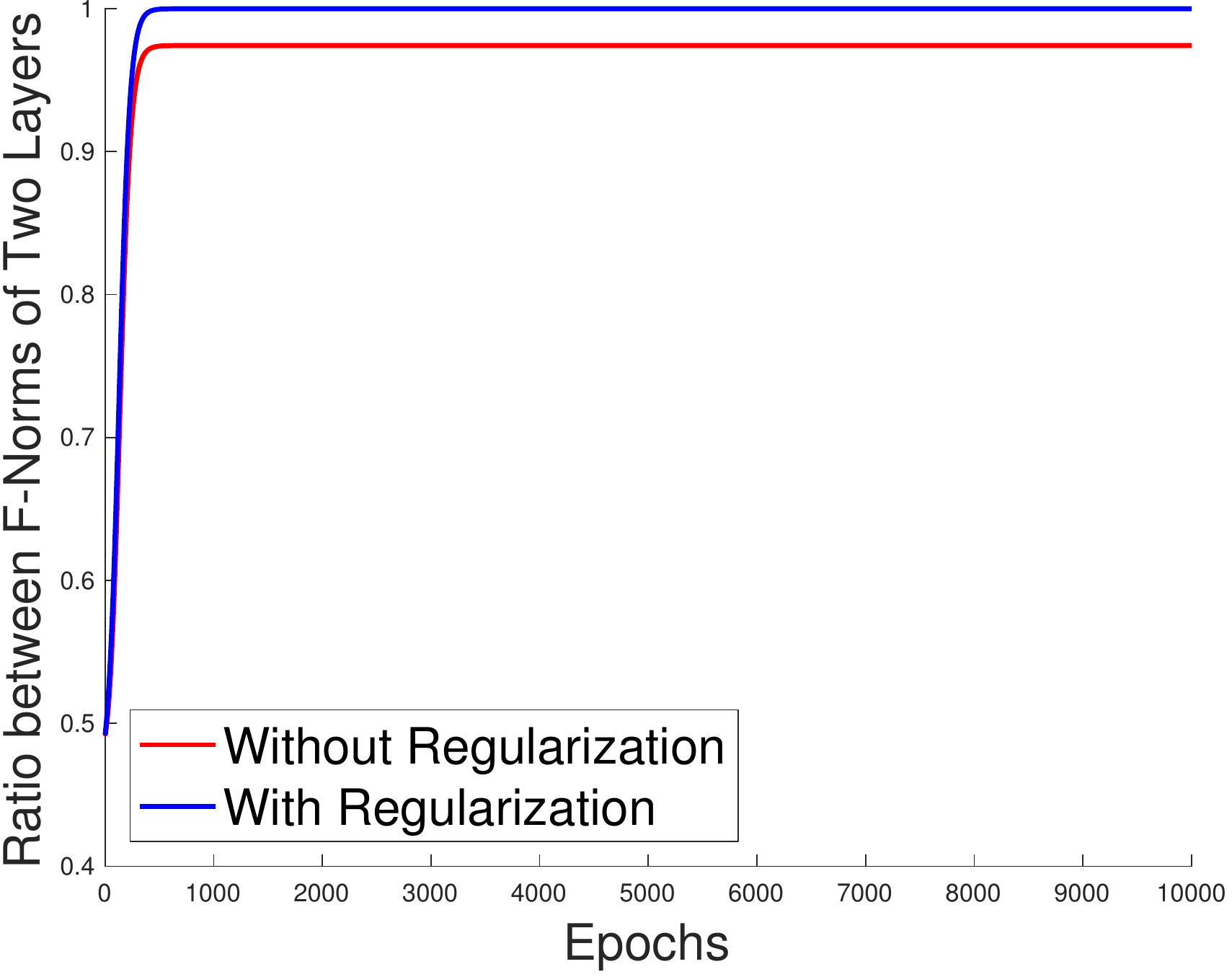}
		\caption{
			Comparison of quantity $\norm{\mat{U}}_F^2 / \norm{\mat{V}}_F^2$ when running GD for objective functions~\eqref{eqn:intro_mf_obj} and~\eqref{eqn:intro_mf_reg_obj}.
		}
					\label{fig:balanced_norm}
	\end{subfigure}
	\caption{Experiments on the matrix factorization problem with objective functions~\eqref{eqn:intro_mf_obj} and~\eqref{eqn:intro_mf_reg_obj}.
		Red lines correspond to running GD on the objective function~\eqref{eqn:intro_mf_obj}, and 
		blue lines correspond to running GD on the objective function~\eqref{eqn:intro_mf_reg_obj}.
	}
	\label{fig:homogeneity}
\end{figure*}

On the other hand, empirically, removing the homogeneity is not necessary.
We use GD with random initialization to solve the optimization problem~\eqref{eqn:intro_mf_obj}.
Figure~\ref{fig:mf_converge} shows that even \emph{without} regularization term like in the modified objective~\eqref{eqn:intro_mf_reg_obj} GD with random initialization converges to a global minimum and the convergence rate is also competitive.
A more interesting phenomenon is shown in Figure~\ref{fig:balanced_norm} in which we track the Frobenius norms of $\mat{U}$ and $\mat{V}$ in all iterations. The plot shows that the ratio between norms remains a constant in all iterations. 
Thus the unbalancedness does not occur at all!
In many practical applications, many models also admit the homogeneous property (like deep neural networks) and first order methods often converge to a balanced solution.
A natural question arises: \begin{center}
\textbf{Why does GD balance multiple layers and converge in learning homogeneous functions?}
\end{center}

In this paper, we take an important step towards answering this question.
Our key finding is that the gradient descent algorithm provides an implicit regularization on the target homogeneous function.
First, we show that on the gradient flow (gradient descent with infinitesimal step size)  trajectory induced by any differentiable loss function, for a large class of homogeneous models, including fully connected and convolutional neural networks with linear, ReLU and Leaky ReLU activations, the differences between squared norms across layers remain invariant.
Thus, as long as at the beginning the differences are small, they remain small at all time.
Note that small differences arise in commonly used initialization schemes such as  $\frac{1}{\sqrt{d}}$ Gaussian initialization or Xavier/Kaiming initialization schemes~\citep{glorot2010understanding,he2016deep}.
%\simon{
Our result thus explains why using ReLU activation is a better choice than sigmoid from the optimization point view.
%}
%Our result sheds light on why first-order methods can optimize deep neural networks well in practice.\simon{something more?}
For linear activation, we prove an even stronger invariance for gradient flow: we show that $\mat{W}^{(h)} (\mat{W}^{(h)})^\top  - (\mat{W}^{(h+1)})^\top \mat{W}^{(h+1)} $ stays invariant over time, where $\mat{W}^{(h)}$ and $\mat{W}^{(h+1)}$ are weight matrices in consecutive layers with linear activation in between.

%$\mat{U}^\top \mat{U} - \mat{V}^\top\mat{V}$ is a constant matrix, which motivates why this regularizer is un-necessary.
%Combined with recent advances in understanding non-convex geometry, we show for asymmetric matrix factorization, gradient flow finds a global minimum without any regularization.

Next, we go beyond gradient flow and consider gradient descent with positive step size.
We focus on the asymmetric matrix factorization problem~\eqref{eqn:intro_mf_obj}.
Our invariance result for linear activation indicates that $\mat{U}^\top \mat{U} - \mat{V}^\top\mat{V}$ stays unchanged for gradient flow.
For gradient descent,  $\mat{U}^\top \mat{U} - \mat{V}^\top\mat{V}$ can change over iterations.
Nevertheless we show that if the step size decreases like $\eta_t = O\left(t^{-\left( \frac12+\delta\right)} \right)$ ($0<\delta\le\frac12$), $\mat{U}^\top \mat{U} - \mat{V}^\top\mat{V}$ will remain small in all iterations.
In the set where $\mat{U}^\top \mat{U} - \mat{V}^\top \mat{V}$ is small, the loss is coercive, and gradient descent thus ensures that all the iterates are bounded.
Using these properties, we then show that gradient descent converges to a globally optimal solution.
Furthermore, for rank-$1$ asymmetric matrix factorization, we give a finer analysis and show that randomly initialized gradient descent with \emph{constant} step size converges to the global minimum at a globally linear rate.

\subsection{Related Work}
\label{sec:rel} %\simon{
	The homogeneity issue has been previously discussed by~\cite{neyshabur2015path,neyshabur2015data}. The authors proposed a variant of stochastic gradient descent that regularizes paths in a neural network, which is related to the max-norm.
The algorithm outperforms gradient descent and AdaGrad on several classification tasks.
%}

A line of research focused on analyzing gradient descent dynamics for (convolutional) neural networks with one or two unknown layers~\citep{tian2017analytical,brutzkus2017globally,du2017convolutional,du2017spurious,zhong2017recovery,li2017convergence,ma2017power,brutzkus2017sgd}.
For one unknown layer, there is no homogeneity issue.
While for two unknown layers, existing work either requires learning two layers separately~\citep{zhong2017recovery,ge2017learning} or uses re-parametrization like weight normalization to remove the homogeneity issue~\citep{du2017spurious}.
To our knowledge, there is no rigorous analysis for optimizing multi-layer homogeneous functions.

%Recent Advances in Non-convex Optimization
%\textbf{Local minimum and Global minimum.} 
For a general (non-convex) optimization problem, it is known that
%Learning a neural network is often formulated as a non-convex problem.
if the objective function satisfies (i) gradient changes smoothly if the parameters are perturbed, 
(ii) all saddle points and local maxima are strict (i.e., there exists a direction with negative curvature), and (iii) all local minima are global (no spurious local minimum), then gradient descent~\citep{lee2016gradient, panageas2016gradient} converges to a global minimum.
There have been many studies on the optimization landscapes of neural networks~\citep{kawaguchi2016deep,choromanska2015loss,du2018power,hardt2016identity,bartlett2018gradient,haeffele2015global,freeman2016topology,vidal2017mathematics,safran2016quality,zhou2017landscape,nguyen2017loss,nguyen2017loss2,zhou2017landscape,safran2017spurious}, showing that the objective functions have properties (ii) and (iii).
Nevertheless, the objective function is in general not smooth as we discussed before.
Our paper complements these results by showing that the magnitudes of all layers are balanced and in many cases, this implies smoothness.
%\simon{do we need to discuss path sgd?  like~\citep{neyshabur2015data,neyshabur2015path}}

\subsection{Paper Organization}
The rest of the paper is organized as follows.
%In Section~\ref{sec:pre}, we review necessary background and notations.
In Section~\ref{sec:conserved}, we present our main theoretical result on the implicit regularization property of gradient flow for optimizing neural networks.
In Section~\ref{sec:mf}, we analyze the dynamics of randomly initialized gradient descent for asymmetric matrix factorization problem with unregularized objective function~\eqref{eqn:intro_mf_obj}.
In Section~\ref{sec:exp}, we empirically verify the theoretical result in Section~\ref{sec:conserved}.
We conclude and list future directions in Section~\ref{sec:con}.
Some technical proofs are deferred to the appendix.

\subsection{Notation}
We use bold-faced letters for vectors and matrices.
For a vector $\vect x$, denote by $\vect x[i]$ its $i$-th coordinate.
For a matrix $\vect A$, we use $\mat A[i, j]$ to denote its $(i, j)$-th entry, and use $\mat A[i, :]$ and $\mat A[:, j]$ to denote its $i$-th row and $j$-th column, respectively (both as column vectors).
We use $\norm{\cdot}_2$ or $\norm{\cdot}$ to denote the Euclidean norm of a vector, and use $\norm{\cdot}_F$ to denote the Frobenius norm of a matrix.
We use $\langle \cdot, \cdot \rangle$ to denote the standard Euclidean inner product between two vectors or two matrices.
Let $[n] = \{1, 2, \ldots, n\}$.

%\section{Preliminaries}
%\label{sec:pre}
%\input{pre.tex}

\section{The Auto-Balancing Properties in Deep Neural Networks}
\label{sec:conserved}
In this section we study the implicit regularization imposed by gradient descent with infinitesimal step size (gradient flow) in training deep neural networks.
In Section~\ref{subsec:conserved-fully-connected} we consider fully connected neural networks, and our main result (Theorem~\ref{thm:conserved-neuron}) shows that gradient flow automatically balances the incoming and outgoing weights at every neuron.
This directly implies that the weights between different layers are balanced (Corollary~\ref{cor:conserved-F-norm}).
For linear activation, we derive a stronger auto-balancing property (Theorem~\ref{thm:conserved-linear}).
In Section~\ref{subsec:cnn} we generalize our result from fully connected neural networks to convolutional neural networks.
In Section~\ref{subsec:proof_main} we present the proof of Theorem~\ref{thm:conserved-neuron}.
The proofs of other theorems in this section follow similar ideas %in the proof of Theorem~\ref{thm:conserved-neuron} 
and are deferred to Appendix~\ref{sec:proof-conserved}.

\subsection{Fully Connected Neural Networks} \label{subsec:conserved-fully-connected}

We first formally define a fully connected feed-forward neural network with $N$ ($N\ge2$) layers.
%, which describes a function $f_{\vect{w}}(\vect{x}) = \mat{W}^{(N)} \phi_{N-1} \left( \mat{W}^{(N-1)}\phi_{N-2}\left( \cdots \mat{W}^{(2)} \phi_1 \left(\mat{W}^{(1)}\vect{x}\right) \cdots \right) \right)$
Let $\mat{W}^{(h)} \in \R^{n_h \times n_{h-1}}$ be the weight matrix in the $h$-th layer, and define $\vect{w} = ( \mat{W}^{(h)} )_{h=1}^N $ as a shorthand of the collection of all the weights.
Then the function $f_{\vect{w}}: \R^d \to \R^p$ ($d = n_0, p = n_N$) computed by this network can be defined recursively: $f_{\vect{w}}^{(1)}(\vect{x}) = \mat{W}^{(1)}\vect{x}$, $f_{\vect{w}}^{(h)}(\vect{x}) = \mat{W}^{(h)} \phi_{h-1}(f_{\vect{w}}^{(h-1)}(\vect{x}))$ ($h = 2, \ldots, N$), and $f_{\vect{w}}(\vect{x}) = f_{\vect{w}}^{(N)}(\vect{x})$, where each $\phi_h$ is an activation function that acts coordinate-wise on vectors.\footnote{We omit the trainable bias weights in the network for simplicity, but our results can be directly generalized to allow bias weights.}
We assume that each $\phi_h$ ($h\in[N-1]$) is \emph{homogeneous}, namely, $\phi_h(x) = \phi_h'(x)\cdot x$ for all $x$ and all elements of the sub-differential $\phi_h'(\cdot) $ when $\phi_h$ is non-differentiable at $x$.
This property is satisfied by functions like ReLU $\phi(x) = \max\{x, 0\}$, Leaky ReLU $\phi(x) = \max\{x, \alpha x\}$ ($0<\alpha<1$), and linear function $\phi(x) = x$.

Let $\ell: \R^p \times \R^p \to \R_{\ge0}$ be a differentiable loss function. Given a training dataset $\left\{ (\vect{x}_i, \vect{y}_i) \right\}_{i=1}^m \subset \R^d\times\R^p$, the training loss as a function of the network parameters $\vect{w}$ is defined as \begin{align}
L(\vect{w}) = \frac1m \sum_{i=1}^m \ell\left( f_{\vect{w}}(\vect{x}_i), \vect{y}_i \right). \label{eqn:nn_loss}
\end{align}

We consider gradient descent with infinitesimal step size (also known as gradient flow) applied on $L(\vect{w})$, which is captured by the differential inclusion:
\begin{equation} \label{eqn:gf-nn}
\frac{\d\mat{W}^{(h)}}{\dt} \in - \frac{\partial L(\vect{w})}{\partial \mat{W}^{(h)}}, \qquad h = 1, \ldots, N,
\end{equation}  
where $t$ is a continuous time index, and $\frac{\partial L(\vect{w})}{\partial \mat{W}^{(h)}}$ is the Clarke sub-differential \citep{clarke2008nonsmooth}. If curves  ${\mat{W}}^{(h)} = {\mat{W}}^{(h)} (t)$ ($h\in[N]$) evolve with time according to \eqref{eqn:gf-nn} they are said to be a solution of the gradient flow differential inclusion.

Our main result in this section is the following invariance imposed by gradient flow.
\begin{thm}[Balanced incoming and outgoing weights at every neuron] \label{thm:conserved-neuron}
	For any $h\in[N-1]$ and $i\in[n_h]$, we have
	\begin{equation} \label{eqn:conserved-neuron}
	\frac{\d}{\dt} \left( \|\mat{W}^{(h)}[i, :]\|^2 - \|\mat{W}^{(h+1)}[:,i]\|^2 \right) = 0.
	\end{equation}
\end{thm}
Note that $\mat{W}^{(h)}[i, :]$ is a vector consisting of network weights coming into the $i$-th neuron in the $h$-th hidden layer, and $\mat{W}^{(h+1)}[:,i]$ is the vector of weights going out from the same neuron. 
 Therefore, Theorem~\ref{thm:conserved-neuron} shows that gradient flow exactly preserves the difference between the squared $\ell_2$-norms of incoming weights and outgoing weights at any neuron.

Taking sum of \eqref{eqn:conserved-neuron} over $i\in[n_h]$, we obtain the following corollary which says gradient flow preserves the difference between the squares of Frobenius norms of weight matrices.
\begin{cor}[Balanced weights across layers] \label{cor:conserved-F-norm}
	For any $h\in[N-1]$, we have
	\begin{equation*}
	\frac{\d}{\dt} \left( \|\mat{W}^{(h)}\|_F^2 - \|\mat{W}^{(h+1)}\|_F^2 \right) = 0.
	\end{equation*}
\end{cor}

Corollary~\ref{cor:conserved-F-norm} explains why in practice, trained multi-layer models usually have similar magnitudes on all the layers:
if we use a small initialization,  $\|\mat{W}^{(h)}\|_F^2 - \|\mat{W}^{(h+1)}\|_F^2$ is very small at the beginning, and Corollary~\ref{cor:conserved-F-norm} implies this difference remains small at all time.
This finding also partially explains why gradient descent converges.
Although the objective function like \eqref{eqn:nn_loss} may not be smooth over the entire parameter space, given that $\|\mat{W}^{(h)}\|_F^2 - \|\mat{W}^{(h+1)}\|_F^2$ is small for all $h$, the objective function may have smoothness.
Under this condition, standard theory shows that gradient descent converges.
We believe this finding serves as a key building block for understanding first order methods for training deep neural networks.

For linear activation, we have the following stronger invariance than Theorem~\ref{thm:conserved-neuron}:
\begin{thm}[Stronger balancedness property for linear activation] \label{thm:conserved-linear}
	If for some $h\in[N-1]$ we have $\phi_h(x) = x$, then
	\begin{equation*}
	\frac{\d}{\dt} \left( \mat{W}^{(h)} (\mat{W}^{(h)})^\top  - (\mat{W}^{(h+1)})^\top \mat{W}^{(h+1)} \right) = \mat{0}.
	\end{equation*}
\end{thm}
This result was known for linear networks \citep{arora2018optimization}, but the proof there relies on the entire network being linear while Theorem~\ref{thm:conserved-linear}  only needs two consecutive layers to have no nonlinear activations in between.

While Theorem~\ref{thm:conserved-neuron} shows the invariance in a node-wise manner, Theorem~\ref{thm:conserved-linear} shows for linear activation, we can derive a layer-wise invariance.
%\simon{anything other discussion?}
Inspired by this strong invariance, in Section~\ref{sec:mf} we prove gradient descent with positive step sizes preserves this invariance approximately for matrix factorization.

\subsection{Convolutional Neural Networks}
\label{subsec:cnn}
Now we show that the conservation property in Corollary~\ref{cor:conserved-F-norm} can be generalized to convolutional neural networks.
In fact, we can allow \emph{arbitrary sparsity pattern and weight sharing structure} within a layer; convolutional layers are a special case.

\paragraph{Neural networks with sparse connections and shared weights.}
We use the same notation as in Section~\ref{subsec:conserved-fully-connected}, with the difference that some weights in a layer can be \emph{missing} or \emph{shared}.
Formally, the weight matrix $\mat W^{(h)} \in \R^{n_h\times n_{h-1}}$ in layer $h$ ($h\in [N]$) can be described by a vector $\vect{v}^{(h)} \in \R^{d_h}$  and a function $g_h: [n_h]\times[n_{h-1}] \to [d_h]\cup\{0\}$.
Here $\vect{v}^{(h)}$ consists of the actual \emph{free parameters} in this layer and $d_h$ is the number of free parameters (e.g. if there are $k$ convolutional filters in layer $h$ each with size $r$, we have $d_h = r\cdot k$).
The map $g_h$ represents the sparsity and weight sharing pattern:
\begin{align*}
\mat W^{(h)}[i, j] = \begin{cases}
0, & g_h(i, j) = 0, \\
\vect{v}^{(h)}[k], & g_h(i, j) =k > 0.
\end{cases}
\end{align*}
Denote by $\vect{v} = \left( \vect{v}^{(h)} \right)_{h=1}^N$ the collection of all the parameters in this network, and we consider gradient flow to learn the parameters:
\begin{equation*}
\frac{\d\vect{v}^{(h)}}{\dt} \in  - \frac{\partial L(\vect{v})}{\partial \vect{v}^{(h)}}, \qquad h = 1, \ldots, N.
\end{equation*}

The following theorem generalizes Corollary~\ref{cor:conserved-F-norm} to neural networks with sparse connections and shared weights:
\begin{thm}\label{thm:cnn}
	For any $h\in[N-1]$, we have
	\begin{equation*}
	\frac{\d}{\dt} \left( \|\vect{v}^{(h)}\|^2 - \|\vect{v}^{(h+1)}\|^2 \right) = 0.
	\end{equation*}
\end{thm}
Therefore, for a neural network with arbitrary sparsity pattern and weight sharing structure, gradient flow still balances the magnitudes of all layers.

\subsection{Proof of Theorem~\ref{thm:conserved-neuron}}
\label{subsec:proof_main}

The proofs of all theorems in this section are similar. They are based on the use of the chain rule (i.e. back-propagation) and the property of homogeneous activations.
Below we provide the proof of Theorem~\ref{thm:conserved-neuron} and defer the proofs of other theorems to Appendix~\ref{sec:proof-conserved}.

\begin{proof}[Proof of Theorem~\ref{thm:conserved-neuron}]
	First we note that we can without loss of generality assume $L$ is the loss associated with one data sample $(\vect{x}, \vect{y}) \in \R^d\times\R^p$, i.e., $L(\vect{w}) = \ell(f_{\vect{w}}(\vect{x}), \vect{y})$.
	In fact, for $L(\vect{w}) = \frac1m \sum_{k=1}^m L_k(\vect{w})$ where $L_k(\vect{w}) = \ell\left( f_{\vect{w}}(\vect{x}_k), \vect{y}_k \right)$, for any single weight $\mat W^{(h)}[i, j]$ in the network we can compute $\frac{\d}{\dt} (\mat{W}^{(h)}[i,j])^2 = 2 \mat{W}^{(h)}[i,j] \cdot \frac{\d \mat{W}^{(h)}[i,j]}{\dt} = -2 \mat{W}^{(h)}[i,j] \cdot \frac{\partial L(\vect{w})}{\partial \mat{W}^{(h)}[i,j]} = -2 \mat{W}^{(h)}[i,j] \cdot \frac1m \sum_{k=1}^m \frac{\partial L_k(\vect{w})}{\partial \mat{W}^{(h)}[i,j]}$, using the sharp chain rule of differential inclusions for tame functions \citep{drusvyatskiy2015curves,davis2018stochastic}.
	Thus, if we can prove the theorem for every individual loss $L_k$, we can prove the theorem for $L$ by taking average over $k\in[m]$.
	
	Therefore in the rest of proof we assume $L(\vect{w}) = \ell(f_{\vect{w}}(\vect{x}), \vect{y})$.
	For convenience, we denote $\vect{x}^{(h)} = f_{\vect{w}}^{(h)}(\vect{x})$ ($h\in[N]$), which is the input to the $h$-th hidden layer of neurons for $h\in[N-1]$ and is the output of the network for $h=N$.
	We also denote $\vect{x}^{(0)} = \vect{x}$ and $\phi_0(x)=x$ ($\forall x$).
	
	Now we prove \eqref{eqn:conserved-neuron}.
	Since $\mat{W}^{(h+1)}[k,i]$ ($k\in[n_{h+1}]$) can only affect $L(\vect{w})$ through $\vect{x}^{(h+1)}[k]$  , we have for $k\in[n_{h+1}]$,
	\begin{equation*}
	\frac{\partial L(\vect{w})}{ \partial \mat{W}^{(h+1)}[k,i] }
	=  \frac{\partial L(\vect{w})}{ \partial \vect{x}^{(h+1)}[k] } \cdot \frac{\partial \vect{x}^{(h+1)}[k]}{\partial \mat{W}^{(h+1)}[k,i]}
	= \frac{\partial L(\vect{w})}{ \partial \vect{x}^{(h+1)}[k] } \cdot \phi_{h}(\vect{x}^{(h)}[i]),
	\end{equation*}
	which can be rewritten as
	\begin{equation*}
	\frac{\partial L(\vect{w})}{ \partial \mat{W}^{(h+1)}[:,i] } = \phi_{h}(\vect{x}^{(h)}[i]) \cdot \frac{\partial L(\vect{w})}{ \partial \vect{x}^{(h+1)} }.
	\end{equation*}
	It follows that
	\begin{equation} \label{eqn:proof-conserved-neuron-1}
	\begin{aligned}
	\frac{\d}{\dt} \|\mat{W}^{(h+1)}[:,i]\|^2
	&= 2 \left\langle \mat{W}^{(h+1)}[:,i], \frac{\d}{\dt} \mat{W}^{(h+1)}[:,i]  \right\rangle = -2 \left\langle \mat{W}^{(h+1)}[:,i], \frac{\partial L(\vect{w})}{ \partial \mat{W}^{(h+1)}[:,i] }   \right\rangle \\
	&%= -2 \left\langle \mat{W}^{(h+1)}[:,i], \phi_{h}(\vect{x}^{(h)}[i]) \cdot \frac{\partial L(\vect{w})}{ \partial \vect{x}^{(h+1)} }  \right\rangle 
	= -2 \phi_{h}(\vect{x}^{(h)}[i]) \cdot \left\langle \mat{W}^{(h+1)}[:,i],  \frac{\partial L(\vect{w})}{ \partial \vect{x}^{(h+1)} }  \right\rangle.
	\end{aligned}
	\end{equation}
	On the other hand, $\mat{W}^{(h)}[i, :]$ only affects $L(\vect{w})$ through $\vect x^{(h)}[i]$. Using the chain rule, we get
	\begin{align*}
	\frac{ \partial L(\vect{w}) }{ \partial \mat{W}^{(h)}[i, :] }
	&= \frac{ \partial L(\vect{w}) }{ \partial \vect x^{(h)}[i] } \cdot \phi_{h-1} (\vect x^{(h-1)})= \left\langle \frac{ \partial L(\vect{w}) }{ \partial \vect x^{(h+1)} } , \mat{W}^{(h+1)}[:, i] \right\rangle \cdot \phi_h'(\vect x^{(h)}[i]) \cdot \phi_{h-1} (\vect x^{(h-1)}),
	\end{align*}
	where $\phi'$ is interpreted as a set-valued mapping whenever it is applied at a non-differentiable point.\footnote{More precisely, the equalities  should be an inclusion whenever there is a sub-differential, but as we see in the next display the ambiguity in the choice of sub-differential does not affect later calculations.}
	
	It follows that\footnote{This holds for any choice of element of the sub-differential, since $\phi'(x) x = \phi(x)$ holds at $x=0$ for any choice of sub-differential.}
	\begin{align*}
&	\frac{\d}{\dt} 	\|\mat{W}^{(h)}[i, :]\|^2
= 2 \left\langle \mat{W}^{(h)}[i, :], \frac{\d}{\dt} \mat{W}^{(h)}[i, :]  \right\rangle
	= -2 \left\langle \mat{W}^{(h)}[i, :], \frac{\partial L(\vect{w})}{ \partial \mat{W}^{(h)}[i, :] }   \right\rangle \\
	 =\,& -2 \left\langle \frac{ \partial L(\vect{w}) }{ \partial \vect x^{(h+1)} } , \mat{W}^{(h+1)}[:, i] \right\rangle \cdot \phi_h'(\vect x^{(h)}[i]) \cdot \left\langle \mat{W}^{(h)}[i, :], \phi_{h-1} (\vect x^{(h-1)}) \right\rangle \\
	=\,& -2 \left\langle \frac{ \partial L(\vect{w}) }{ \partial \vect x^{(h+1)} } , \mat{W}^{(h+1)}[:, i] \right\rangle \cdot \phi_h'(\vect x^{(h)}[i]) \cdot \vect x^{(h)}[i] = -2 \left\langle \frac{ \partial L(\vect{w}) }{ \partial \vect x^{(h+1)} } , \mat{W}^{(h+1)}[:, i] \right\rangle \cdot \phi_{h}(\vect{x}^{(h)}[i]).
	\end{align*}
	Comparing the above expression to \eqref{eqn:proof-conserved-neuron-1}, we finish the proof.
\end{proof}

\section{Gradient Descent Converges to Global Minimum for Asymmetric Matrix Factorization}
\label{sec:mf}

In this section we constrain ourselves to the asymmetric matrix factorization problem and analyze the gradient descent algorithm with random initialization.
Our analysis is inspired by the auto-balancing properties presented in Section~\ref{sec:conserved}. We extend these properties from gradient flow to gradient descent with positive step size.

Formally, we study the following non-convex optimization problem:
\begin{equation} \label{eqn:mf}
\min_{\mat{U} \in \R^{d_1\times r}, \mat{V} \in \R^{d_2\times r}} f(\mat{U}, \mat{V}) = \frac12 \norm{\mat{U}\mat{V}^\top - \mat{M}^*}_F^2,
\end{equation}
where $\mat M^* \in \R^{d_1\times d_2}$ has rank $r$.
Note that we do not have any explicit regularization in \eqref{eqn:mf}.
The gradient descent dynamics for \eqref{eqn:mf} have the following form:
\begin{equation} \label{eqn:mf-gd-dynamics}
\mat U_{t+1} = \mat U_t - \eta_t (\mat U_t \mat V_t^\top - \mat M^*) \mat V_t, \qquad
\mat V_{t+1} = \mat V_t - \eta_t (\mat U_t \mat V_t^\top - \mat M^*)^\top \mat U_t.
\end{equation}

\subsection{The General Rank-$r$ Case}

First we consider the general case of $r\ge1$.
Our main theorem below says that if we use a random small initialization $(\mat U_0, \mat V_0)$, and set step sizes $\eta_t$ to be appropriately small, then gradient descent \eqref{eqn:mf-gd-dynamics} will converge to a solution close to the global minimum of \eqref{eqn:mf}.
To our knowledge, this is the first result showing that gradient descent with random initialization directly solves the un-regularized asymmetric matrix factorization problem~\eqref{eqn:mf}.

\begin{thm}\label{thm:mf-main}
	Let $0<\epsilon < \norm{\mat M^*}_F$.
	Suppose we initialize the entries in $\mat U_0$ and $\mat V_0$ i.i.d. from $\cN(0, \frac{\epsilon}{\poly(d)})$ ($d = \max\{d_1, d_2\}$), and run  \eqref{eqn:mf-gd-dynamics} with step sizes $\eta_t = \frac{\sqrt{\epsilon/r}}{100(t+1) \norm{\mat M^*}_F^{3/2}}$ ($t=0,1,\ldots$).\footnote{The dependency of $\eta_t$ on $t$ can be $\eta_t = \Theta\left( t^{-(1/2+\delta)} \right)$ for any constant $\delta \in (0, 1/2]$.}
	Then with high probability over the initialization, $\lim_{t\to\infty}(\mat U_t, \mat V_t) = (\bar{\mat U}, \bar{\mat V})$ exists and satisfies $\norm{\bar{\mat U} \bar{\mat V}^\top - \mat M^*}_F \le \epsilon$.
\end{thm}

%The full proof of Theorem~\ref{thm:mf-main} is given in Appendix~\ref{}. Now we give a sketch of the proof.

\paragraph{Proof sketch of Theorem~\ref{thm:mf-main}.}
	First let's imagine that we are using infinitesimal step size in GD. Then according to Theorem~\ref{thm:conserved-linear} (viewing problem~\eqref{eqn:mf} as learning a two-layer linear network where the inputs are all the standard unit vectors in $\R^{d_2}$), we know that $\mat U^\top \mat U - \mat V^\top \mat V$ will stay invariant throughout the algorithm.
	Hence when $\mat U$ and $\mat V$ are initialized to be small, $\mat U^\top \mat U - \mat V^\top \mat V$ will stay small forever.
	Combined with the fact that the objective $f(\mat U, \mat V)$ is decreasing over time (which means $\mat U \mat V^\top$ cannot be too far from $\mat M^*$), we can show that $\mat U$ and $\mat V$ will always stay bounded.
	
	Now we are using positive step sizes $\eta_t$, so we no longer have the invariance of $\mat U^\top \mat U - \mat V^\top \mat V$.
	Nevertheless, by a careful analysis of the updates, we can still prove that $\mat U_t^\top \mat U_t - \mat V_t^\top \mat V_t$ is small, the objective $f(\mat U_t, \mat V_t)$ decreases, and $\mat U_t$ and $\mat V_t$ stay bounded.
	Formally, we have the following lemma:
	\begin{lem} \label{lem:mf-balance}
		With high probability over the initialization $(\mat U_0, \mat V_0)$, for all $t$ we have:
		\begin{enumerate}[(i)]
			\item Balancedness: $\norm{\mat U_t^\top \mat U_t - \mat V_t^\top \mat V_t}_F \le \epsilon$;
			\item Decreasing objective: $f(\mat U_{t}, \mat V_{t}) \le f(\mat U_{t-1}, \mat V_{t-1}) \le \cdots \le f(\mat U_{0}, \mat V_{0}) \le 2\norm{\mat M^*}_F^2$;
			\item Boundedness: $\norm{\mat U_{t}}_F^2 \le 5\sqrt{r} \norm{\mat M^*}_F, \norm{\mat V_t}_F^2 \le  5\sqrt{r} \norm{\mat M^*}_F$.
		\end{enumerate}
	\end{lem}
	
	Now that we know the GD algorithm automatically constrains $(\mat U_t, \mat V_t)$ in a bounded region, we can use the smoothness of $f$ in this region and a standard analysis of GD to show that $(\mat U_t, \mat V_t)$ converges to a stationary point $(\bar{\mat U}, \bar{\mat V})$ of $f$ (Lemma~\ref{lem:mf-convergence}).
	Furthermore, using the results of \citep{lee2016gradient, panageas2016gradient} we know that $(\bar{\mat U}, \bar{\mat V})$ is almost surely not a strict saddle point.
	Then the following lemma implies that $(\bar{\mat U}, \bar{\mat V})$ has to be close to a global optimum since we know $\norm{\bar{\mat U}^\top \bar{\mat U} - \bar{\mat V}^\top \bar{\mat V}}_F \le \epsilon$ from Lemma~\ref{lem:mf-balance} (i). This would complete the proof of Theorem~\ref{thm:mf-main}.

\begin{lem} \label{lem:mf-strict-saddle}
	Suppose $(\mat{U}, \mat{V})$ is a stationary point of $f$ such that $\norm{\mat U^\top \mat U - \mat V^\top \mat V}_F \le \epsilon$.
	Then either $\norm{\mat U \mat V^\top - \mat M^*}_F \le \epsilon$, or $(\mat{U}, \mat{V})$ is a strict saddle point of $f$.
\end{lem}

The full proof of Theorem~\ref{thm:mf-main} and the proofs of Lemmas~\ref{lem:mf-balance} and \ref{lem:mf-strict-saddle} are given in Appendix~\ref{sec:proof-mf}.

\subsection{The Rank-$1$ Case}

%In previous sections, we have shown gradient flow automatically balances different layers and this fact implies convergence to global minima in certain problems without any regularization.

We have shown in Theorem~\ref{thm:mf-main} that GD with small and diminishing step sizes converges to a global minimum for matrix factorization.
Empirically, it is observed that a constant step size $\eta_t \equiv \eta$ is enough for GD to converge quickly to global minimum.
%In practice, one uses gradient descent with positive step size and the invariance we derived in Section~\ref{sec:conserved} no longer holds.
Therefore,
some natural questions are how to prove convergence of GD with a constant step size, how fast it converges, and how the discretization affects the invariance we derived in Section~\ref{sec:conserved}.

%Unlike gradient flow, in this setting we cannot rely on the differential equations we used for the proof in the previous section and we need detailed study of the gradient descent dynamics.
%Furthermore, different problems may require different analysis techniques.

While these questions remain challenging for the general rank-$r$ matrix factorization, we resolve them for the case of $r=1$.
%In this section, we initiate the study of gradient descent for learning homogeneous functions. 
%We give a detailed analysis of trajectory of gradient descent with constant positive step size for solving rank-$1$ asymmetric matrix factorization problem.
Our main finding is that with constant step size, the norms of two layers are always within a constant factor of each other (although we may no longer have the stronger balancedness property as in Lemma~\ref{lem:mf-balance}), and 
 we utilize this property to prove the \emph{linear convergence} of GD to a global minimum.

When $r=1$, the asymmetric matrix factorization problem and its GD dynamics become
$$
\min_{\vect{u} \in \mathbb{R}^{d_1}, \vect{v} \in \mathbb{R}^{d_2}} \frac12 \norm{\vect{u}\vect{v}^\top - \mat{M}^*}_F^2
$$
and
\begin{align*}
\vect{u}_{t+1} = \vect{u}_t - \eta (\vect{u}_t\vect{v}_t^\top - \mat{M}^*)\vect{v}_t, \qquad
\vect{v}_{t+1} = \vect{v}_t - \eta\left(\vect{v}_t\vect{u}_t^\top - \mat{M}^{*\top} \right)\vect{u}_t.
\end{align*}
Here we assume $\mat{M}^*$ has rank $1$, i.e., it can be factorized as $\mat{M}^* = \sigma_1 \vect{u}^*\vect{v}^{*\top}$ where $\vect{u}^*$ and $\vect{v}^*$ are unit vectors and $\sigma_1>0$.
%For ease of presentation we assume $d_1 = \Theta (d_2)$.% and denote $d=\max\left\{d_1,d_2\right\}$.

Our main theoretical result is the following.
\begin{thm}[Approximate balancedness and linear convergence of GD for rank-$1$ matrix factorization]
	\label{thm:rank_1}
	Suppose $\vect{u}_0 \sim \cN(\vect 0,\delta\mat{I})$, $\vect{v}_0 \sim \cN(\vect 0,\delta \mat I)$ with $\delta = c_{init} \sqrt{\frac{\sigma_1}{d}}  $ ($d = \max\{d_1, d_2\}$) %, $\delta_2 =c_{init} \sqrt{\sigma}_1/\sqrt{d}$
	 for some sufficiently small constant $c_{init} >0$,
	 %$0 < c_{init} < \frac{1}{100}$, 
	 and $\eta = \frac{c_{step}}{\sigma_1}$ for some sufficiently small constant $ c_{step} >0$. % $0 < c_{step} < \frac{1}{10}$.
	Then with constant probability over the initialization, for all $t$ we have $c_0\le \frac{\abs{\vect{u}_t^\top \vect u^*}}{\abs{\vect{v}_t^\top \vect{v}^*}}\le C_0$ for some universal constants $ c_0,C_0>0$.
	Furthermore, for any $0<\epsilon< 1$, after $t = O\left( \log\frac{d}{\epsilon} \right)$ iterations, we have $\norm{\vect{u}_t\vect{v}_t^\top - \mat{M}^*}_F \le \epsilon \sigma_1$.
\end{thm}
Theorem~\ref{thm:rank_1} shows for $\vect{u}_t$ and $\vect{v}_t$, their strengths in the signal space, $\abs{\vect{u}_t^\top \vect{u}^*}$ and $\abs{\vect{v}_t^\top \vect{v}^*}$, are of the same order.
%This fact can be viewed as a generalization of Corollary~\ref{cor:conserved-F-norm} where here this property only approximately holds.
This approximate balancedness helps us prove the linear convergence of GD.
%Since two layer are balanced, we can choose a positive constant step size $\eta \asymp \frac{1}{\sigma_1}$ for all iterations and this leads to linear convergence.
%The main idea of the proof is to carefully study the dynmics of error $\norm{\vect{u}\vect{v}^\top -\mat{M}^*}_F^2$ and the ratio $\frac{\abs{\vect{u}_t^\top \vect{u}^*}}{\abs{\vect{v}_t^\top \vect{v}^*}}$.
We refer readers to Appendix~\ref{sec:proof-rank-1} for the proof of Theorem~\ref{thm:rank_1}.

%\section{Analysis of Gradient Descent with Constant Step Size for Rank-$1$ Asymmetric Matrix Factorization}
%\section{GD Converges Linearly to Global Minimum in Rank-$1$ Asymmetric Matrix Factorization}
%\label{sec:rank_1}
%\input{rank_1.tex}

\section{Empirical Verification}
\label{sec:exp}

We perform experiments to verify the auto-balancing properties of gradient descent in neural networks with ReLU activation.
Our results below show that for GD with small step size and small initialization: (1) the difference between the squared Frobenius norms of any two layers remains small in all iterations, and (2)  the ratio between the squared Frobenius norms of any two layers becomes close to $1$.
Notice that our theorems in Section~\ref{sec:conserved} hold for gradient flow (step size $\rightarrow0$) but in practice we can only choose a (small) positive step size, so we cannot hope the difference between the squared Frobenius norms to remain exactly the same but can only hope to observe that the differences remain small.

We consider a 3-layer fully connected network of the form $f(x)= W_3 \phi(W_2\phi(W_1 x))$ where $x \in \mathbb{R}^{1\text{,}000}$ is the input, $W_1 \in \mathbb{R}^{100 \times 1\text{,}000}$, $W_2 \in \mathbb{R}^{100 \times 100}$, $W_3 \in \mathbb{R}^{10 \times 100}$, and $\phi(\cdot)$ is ReLU activation.
We use 1,000 data points and the quadratic loss function,
%choose loss function (Equation (5)) to be quadratic loss.
and run  GD. % and we note after $T=2000$ iterations, the differences loss between two consecutive is smaller than $1e-6$ so the gradient descent algorithm already approximately converges.
We first test a balanced initialization: $W_1[i,j] \sim  N(0,\frac{10^{-4}}{100})$, $W_2[i,j] \sim N(0,\frac{10^{-4}}{10})$ and $W_3[i,j] \sim   N(0,10^{-4})$,
%We choose difference variances for initialization to 
which ensures $\|W_1\|_F^2\approx \|W_2\|_F^2 \approx \|W_3\|_F^2$. 
%$\mathbb{E}\left[\|W_1\|_F^2\right] = \mathbb{E}\left[\|W_2\|_F^2\right] =\mathbb{E}\left[\|W_3\|_F^2\right]$. 
%The scalar $10^{-4}$ is used to ensure small initialization.
After 10,000 iterations we have $\|W_1\|_F^2 = 42.90$, $\|W_2\|_F^2 = 43.76$ and $\|W_3\|_F^2 = 43.68$.
Figure~\ref{fig:ba_norm_dff} shows that in all iterations $\left|\|W_1\|_F^2-\|W_2\|_F^2\right|$ and $\left|\|W_2\|_F^2-\|W_3\|_F^2\right|$ are bounded by $0.14$ which is much smaller than the magnitude of each $\|W_h\|_F^2$.
Figures~\ref{fig:ba_norm_ratios} shows that the ratios between norms approach $1$.
We then test an unbalanced initialization: $W_1[i,j] \sim  N(0,10^{-4})$, $W_2[i,j] \sim N(0,10^{-4})$ and $W_3[i,j] \sim   N(0,10^{-4})$.
After 10,000 iterations we have $\|W_1\|_F^2 = 55.50$, $\|W_2\|_F^2 = 45.65$ and $\|W_3\|_F^2 = 45.46$. 
Figure~\ref{fig:imba_norm_dff} shows that $\left|\|W_1\|_F^2-\|W_2\|_F^2\right|$ and $\left|\|W_2\|_F^2-\|W_3\|_F^2\right|$ are bounded by $9$ (and indeed change very little throughout the process), and Figures~\ref{fig:imba_norm_ratios} shows that the ratios become close to  $1$ after about 1,000 iterations.

\begin{figure*}[t!]
	\centering
	\begin{subfigure}[t]{0.22\textwidth}
		\includegraphics[width=\textwidth]{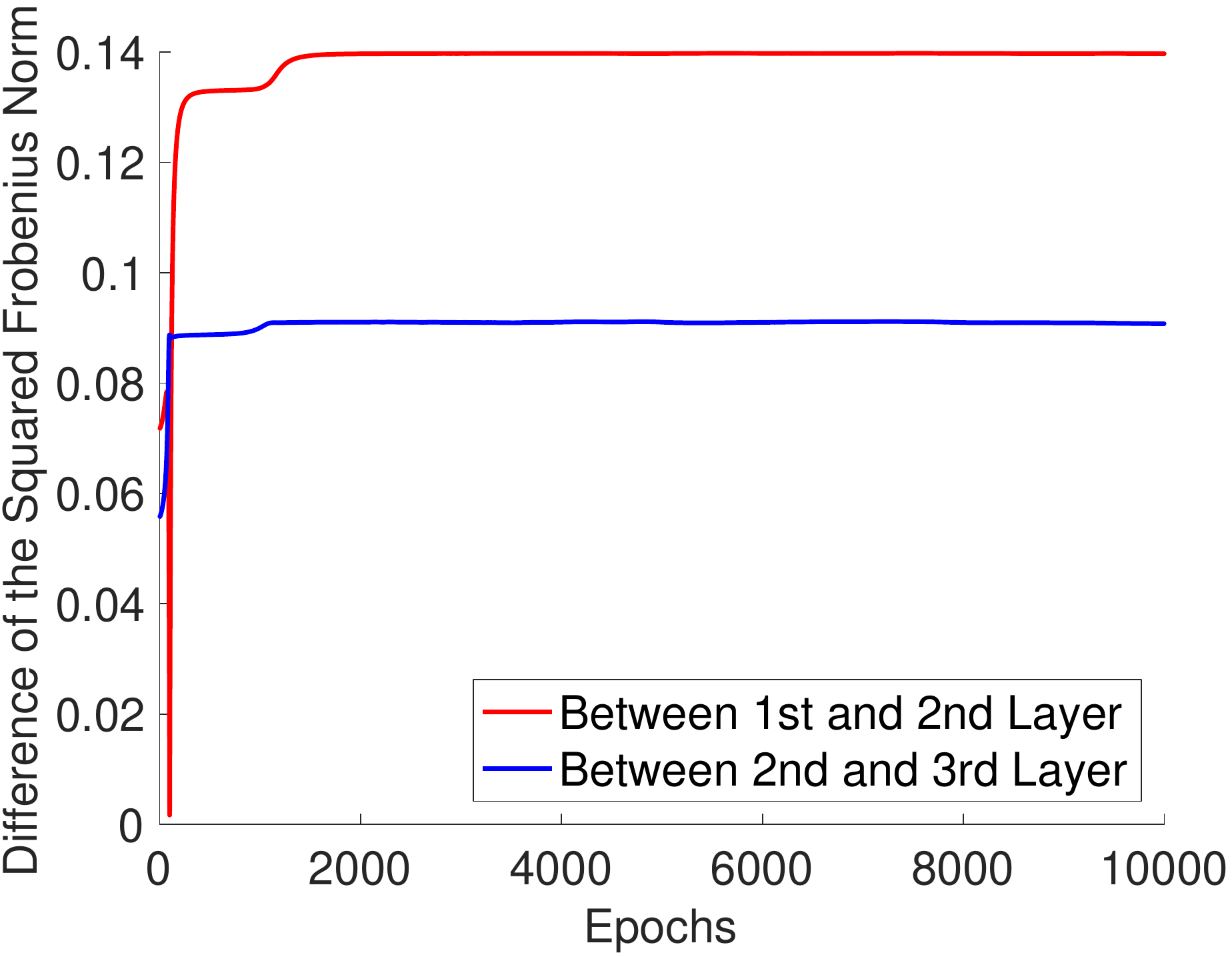}
		\caption{Balanced initialization, squared norm differences.}
		\label{fig:ba_norm_dff}
	\end{subfigure}	
	\quad
	\begin{subfigure}[t]{0.22\textwidth}
		\includegraphics[width=\textwidth]{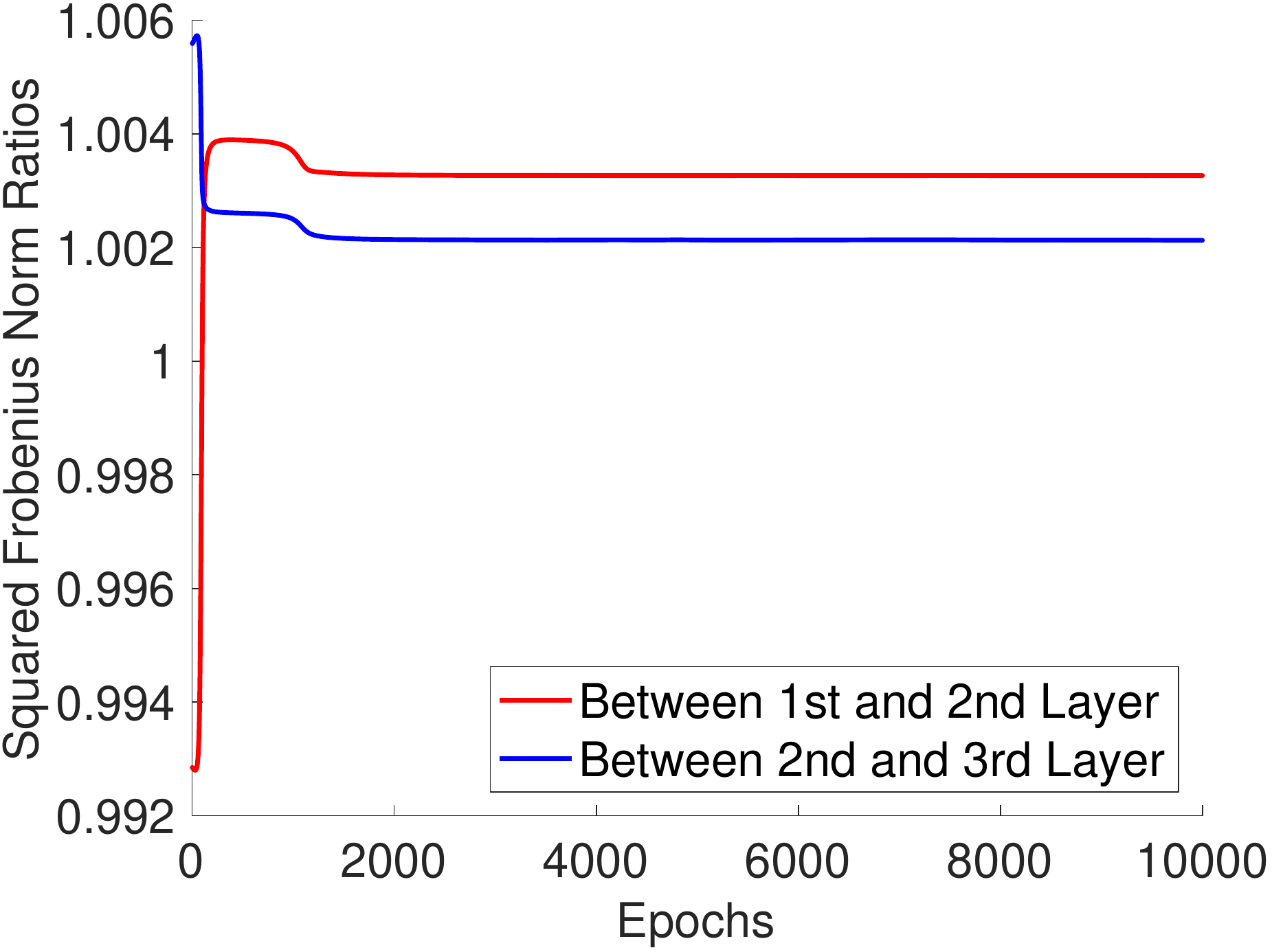}
		\caption{Balanced initialization,  squared norm ratios.}
		\label{fig:ba_norm_ratios}
	\end{subfigure}
	\quad
	\begin{subfigure}[t]{0.22\textwidth}
		\includegraphics[width=\textwidth]{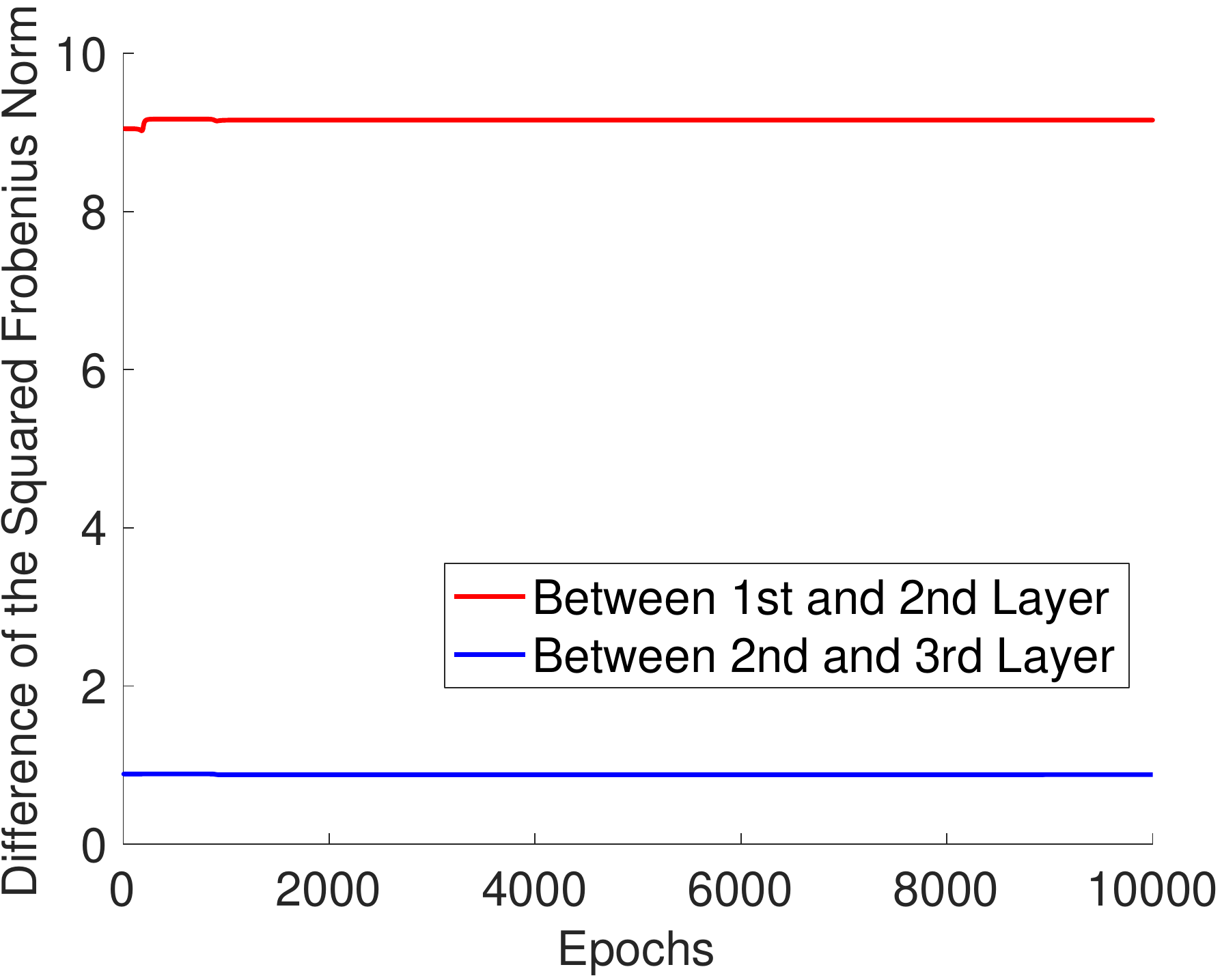}
		\caption{Unbalanced Initialization, squared norm differences.}
		\label{fig:imba_norm_dff}
	\end{subfigure}	
	\quad
	\begin{subfigure}[t]{0.22\textwidth}
		\includegraphics[width=\textwidth]{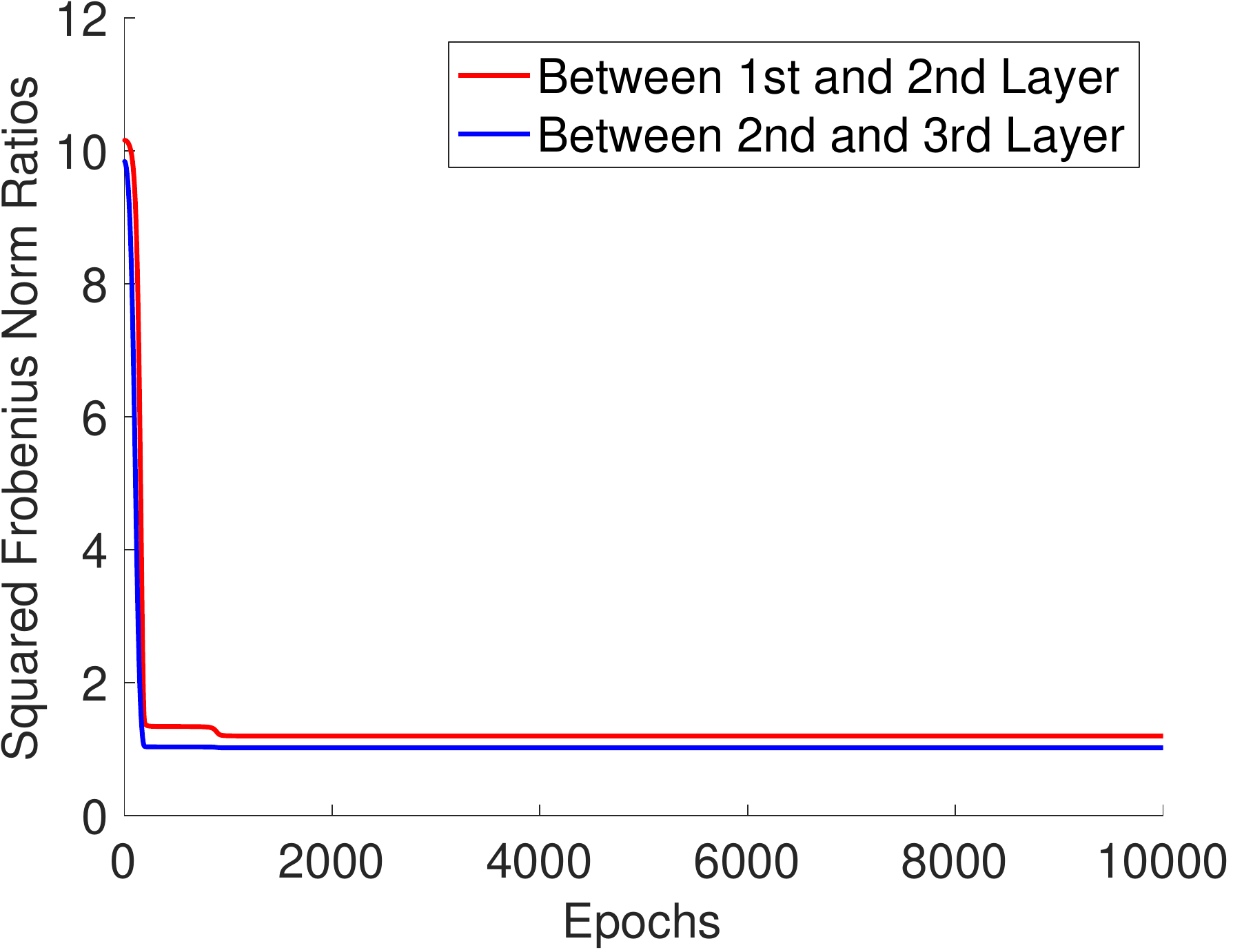}
		\caption{Unbalanced initialization,  squared norm ratios.}
		\label{fig:imba_norm_ratios}
	\end{subfigure}
	\caption{Balancedness of a 3-layer neural network.}
	\label{fig:experiment}
\end{figure*}

\section{Conclusion and Future Work}
\label{sec:con}

In this paper we take a step towards characterizing the invariance imposed by first order algorithms.
We show that gradient flow automatically balances the magnitudes of all layers in a deep neural network with homogeneous activations. 
For the concrete model of asymmetric matrix factorization, we further use the balancedness property to show that gradient descent converges to global minimum.
We believe our findings on the invariance in deep models could serve as a fundamental building block for understanding optimization in deep learning.
Below we list some future directions.

\paragraph{Other first-order methods.}
In this paper we focus on the invariance induced by gradient descent.
In practice, different acceleration and adaptive methods are also used.
A natural future direction is how to characterize the invariance properties of these algorithms.
%These algorithms also induce a set of differential equations like \eqref{eqn:gf-nn}. 
%Therefore our technique can still be applied to study the invariance imposed by these algorithms.

\paragraph{From gradient flow to gradient descent: a generic analysis?}
As discussed in Section~\ref{sec:mf}, while strong invariance properties hold for gradient flow, in practice one uses gradient descent with positive step sizes and the invariance may only hold approximately because positive step sizes discretize the dynamics.
We use specialized techniques for analyzing asymmetric matrix factorization.
It would be very interesting to develop a generic approach to analyze the discretization.
Recent findings on the connection between optimization and ordinary differential equations~\citep{su2014differential,zhang2018direct} might be useful for this purpose.

\section*{Acknowledgements}
\label{sec:ack}
We thank Phil Long for his helpful comments on an earlier draft of this paper. JDL acknowledges support from ARO W911NF-11-1-0303.

\bibliography{simonduref}
\bibliographystyle{plainnat}
\newpage

\appendix
\section*{\Large Appendix}
\section{Proofs for Section~\ref{sec:conserved}}
\label{sec:proof-conserved}
\begin{proof}[Proof of Theorem~\ref{thm:conserved-linear}]
	Same as the proof of Theorem~\ref{thm:conserved-neuron}, we assume without loss of generality that $L(\vect{w}) = \ell(f_{\vect{w}}(\vect{x}), \vect{y})$ for some $(\vect{x}, \vect{y}) \in \R^d\times\R^p$.
	We also denote $\vect{x}^{(h)} = f_{\vect{w}}^{(h)}(\vect{x})$ ($\forall h\in[N]$), $\vect x^{(0)} = \vect x$ and $\phi_0(x) = x$.
	
	Now we suppose $\phi_h(x) = x$ for some $h\in[N-1]$.
	Denote $\vect{u} = \phi_{h-1}(\vect{x}^{(h-1)})$. Then we have $\vect{x}^{(h+1)} = \mat{W}^{(h+1)} \vect{x}^{(h)} = \mat{W}^{(h+1)} \mat{W}^{(h)} \vect{u}$.
	Using the chain rule, we can directly compute
	\begin{align*}
	\frac{\partial L(\vect{w})}{\partial \mat{W}^{(h)}}
	&= \frac{\partial L(\vect{w})}{\partial \vect{x}^{(h)}} \vect{u}^\top
	= (\mat{W}^{(h+1)})^\top \frac{\partial L(\vect{w})}{\partial \vect{x}^{(h+1)}} \vect{u}^\top,\\
	\frac{\partial L(\vect{w})}{\partial \mat{W}^{(h+1)}}
	&= \frac{\partial L(\vect{w})}{\partial \vect{x}^{(h+1)}} (\vect{x}^{(h)})^\top
	= \frac{\partial L(\vect{w})}{\partial \vect{x}^{(h+1)}} ( \mat{W}^{(h)} \vect{u})^\top.
	\end{align*}
	Then we have
	\begin{align*}
		\frac{\d}{\dt} \left( \mat{W}^{(h)} (\mat{W}^{(h)})^\top \right)
		&= \mat{W}^{(h)} \left( \frac{\d}{\dt} \mat{W}^{(h)} \right)^\top + \left( \frac{\d}{\dt} \mat{W}^{(h)} \right) (\mat{W}^{(h)})^\top \\
		&= \mat{W}^{(h)} \vect{u} \left(\frac{\partial L(\vect{w})}{\partial \vect{x}^{(h+1)}}\right)^\top \mat{W}^{(h+1)} + (\mat{W}^{(h+1)})^\top \frac{\partial L(\vect{w})}{\partial \vect{x}^{(h+1)}} \vect{u}^\top (\mat{W}^{(h)})^\top,\\
		\frac{\d}{\dt} \left( (\mat{W}^{(h+1)})^\top \mat{W}^{(h+1)} \right)
		&= (\mat{W}^{(h+1)})^\top \left( \frac{\d}{\dt} \mat{W}^{(h+1)} \right) + \left( \frac{\d}{\dt} \mat{W}^{(h+1)} \right)^\top \mat{W}^{(h+1)} \\
		&= (\mat{W}^{(h+1)})^\top \frac{\partial L(\vect{w})}{\partial \vect{x}^{(h+1)}} \vect{u}^\top (\mat{W}^{(h)})^\top +  \mat{W}^{(h)} \vect{u} \left(\frac{\partial L(\vect{w})}{\partial \vect{x}^{(h+1)}}\right)^\top \mat{W}^{(h+1)}. 
	\end{align*}
	Comparing the above two equations we know $\frac{\d}{\dt} \left( \mat{W}^{(h)} (\mat{W}^{(h)})^\top  - (\mat{W}^{(h+1)})^\top \mat{W}^{(h+1)} \right) = \mat{0}$.
\end{proof}

\begin{proof}[Proof of Theorem~\ref{thm:cnn}]
	Same as the proof of Theorem~\ref{thm:conserved-neuron}, we assume without loss of generality that $L(\vect{v}) = L(\vect{w}) = \ell(f_{\vect{w}}(\vect{x}), \vect{y})$ for $(\vect{x}, \vect{y}) \in \R^d\times\R^p$, and denote $\vect{x}^{(h)} = f_{\vect{w}}^{(h)}(\vect{x})$ ($\forall h\in[N]$), $\vect x^{(0)} = \vect x$ and $\phi_0(x) = x$.
	
	Using the chain rule, we have
	\begin{align*}
	\frac{\partial L(\vect v)}{\partial \vect v^{(h+1)}[l]}
	= \sum_{(k, i): g_{h+1}(k, i) = l} \frac{\partial L(\vect v)}{\partial \vect x^{(h+1)}[k]} \cdot \phi_h(\vect x^{(h)}[i]), \qquad l\in[d_{h+1}].
	\end{align*}
	Then we have using the sharp chain rule,
	\begin{equation} \label{eqn:proof-conserved-cnn-1}
	\begin{aligned}
	\frac{\d}{\dt}  \|\vect{v}^{(h+1)}\|^2
	&= 2 \left\langle \vect{v}^{(h+1)}, \frac{\d}{\dt} \vect{v}^{(h+1)} \right\rangle
	= -2 \left\langle \vect{v}^{(h+1)}, \frac{\partial L(\vect v)}{\partial \vect v^{(h+1)}} \right\rangle \\
	&= -2 \sum_l \sum_{(k, i): g_{h+1}(k, i) = l}  \frac{\partial L(\vect v)}{\partial \vect x^{(h+1)}[k]} \cdot \vect{v}^{(h+1)}[l] \cdot \phi_h(\vect x^{(h)}[i]) \\
	&= -2  \sum_{(k, i)}  \frac{\partial L(\vect v)}{\partial \vect x^{(h+1)}[k]} \cdot \mat{W}^{(h+1)}[k, i] \cdot \phi_h(\vect x^{(h)}[i]) \\
	&= -2  \sum_{k}  \frac{\partial L(\vect v)}{\partial \vect x^{(h+1)}[k]} \cdot \vect x^{(h+1)}[k] \\
	&= -2 \left\langle \frac{\partial L(\vect v)}{\partial \vect x^{(h+1)}} , \vect x^{(h+1)} \right\rangle.
	\end{aligned}
	\end{equation}
	Substituting $h$ with $h-1$ in \eqref{eqn:proof-conserved-cnn-1} gives $\frac{\d}{\dt}  \|\vect{v}^{(h)}\|^2 = -2 \left\langle \frac{\partial L(\vect v)}{\partial \vect x^{(h)}} , \vect x^{(h)} \right\rangle$, which further implies
	\begin{equation}\label{eqn:proof-conserved-cnn-2}
	\begin{aligned}
	\frac{\d}{\dt}  \|\vect{v}^{(h)}\|^2 &= -2 \left\langle \frac{\partial L(\vect v)}{\partial \vect x^{(h)}} , \vect x^{(h)} \right\rangle
	= -2 \sum_{i} \frac{\partial L(\vect v)}{\partial \vect x^{(h)}[i]} \cdot \vect x^{(h)}[i] \\
	&= -2 \sum_{i} \sum_k \frac{\partial L(\vect v)}{\partial \vect x^{(h+1)}[k]}\cdot \mat{W}^{(h+1)}[k, i] \cdot \phi_h'(\vect x^{(h)}[i]) \cdot \vect x^{(h)}[i] \\
	&= -2 \sum_k \frac{\partial L(\vect v)}{\partial \vect x^{(h+1)}[k]} \sum_{i} \mat{W}^{(h+1)}[k, i] \cdot \phi_h(\vect x^{(h)}[i]) \\
	&= -2  \sum_{k}  \frac{\partial L(\vect v)}{\partial \vect x^{(h+1)}[k]} \cdot \vect x^{(h+1)}[k] \\
	&= -2 \left\langle \frac{\partial L(\vect v)}{\partial \vect x^{(h+1)}} , \vect x^{(h+1)} \right\rangle.
	\end{aligned}
	\end{equation}
	The proof is finished by combining \eqref{eqn:proof-conserved-cnn-1} and \eqref{eqn:proof-conserved-cnn-2}.
\end{proof}

%\section{Proofs for Section~\ref{sec:mf}}
\section{Proof for Rank-$r$ Matrix Factorization (Theorem~\ref{thm:mf-main})}
\label{sec:proof-mf}

In this section we give the full proof of Theorem~\ref{thm:mf-main}.

First we recall the gradient of our objective function $f(\mat{U}, \mat{V}) = \frac12 \norm{\mat{U}\mat{V}^\top - \mat{M}^*}_F^2$:
\begin{align*}
\frac{\partial f(\mat U, \mat V)}{\partial \mat U} = (\mat U \mat V^\top - \mat M^*) \mat V, \qquad
\frac{\partial f(\mat U, \mat V)}{\partial \mat V} = (\mat U \mat V^\top - \mat M^*)^\top \mat U.
\end{align*}

We also need to calculate the Hessian $\nabla^2 f(\mat U, \mat V)$.
The Hessian can be viewed as a matrix that operates on vectorized matrices of dimension $(d_1+d_2)\times r$ (i.e., the same shape as $\begin{pmatrix}
\mat U\\ \mat V
\end{pmatrix}$).
Then, for any $\mat W \in \R^{(d_1+d_2)\times r}$, the Hessian $\nabla^2 f(\mat W)$ defines a quadratic form
\[
[\nabla^2 f(\mat W)] (\mat A, \mat B) = \sum_{i, j, k, l} \frac{\partial^2 f(\mat W)}{\partial \mat W[i, j] \partial \mat W[k, l]} \mat A[i, j] \mat B[k, l], \qquad \forall \mat A, \mat B \in \R^{(d_1+d_2)\times r}.
\]
With this notation, we can express the Hessian $\nabla^2 f(\mat U, \mat V)$ as follows:
\begin{equation} \label{eqn:mf-hessian}
\begin{aligned}  \
[\nabla^2 f(\mat U, \mat V)] (\mat \Delta, \mat \Delta)
= 2 \left\langle \mat U \mat V^\top - \mat M^*, \mat\Delta_{\mat U} \mat{\Delta}_{\mat V}^\top \right\rangle + \norm{\mat U \mat\Delta_{\mat V}^\top + \mat\Delta_{\mat U} \mat V^\top }_F^2, &
\\
\forall \mat \Delta = \begin{pmatrix}
\mat \Delta_{\mat U} \\ \mat{\Delta}_{\mat V}
\end{pmatrix}, 
\mat \Delta_{\mat U} \in \R^{d_1\times r}, \ &\mat \Delta_{\mat V} \in \R^{d_2\times r}.
\end{aligned}
\end{equation}

Now we use the expression of the Hessian to prove that $f(\mat U, \mat V)$ is locally smooth when both arguments $\mat U$ and $\mat V$ are bounded.

\begin{lem}[Smoothness over a bounded set]\label{lem:mf-local-smooth}
	For any $c>0$, constrained on the set $\cS = \{(\mat U, \mat V): \mat U \in \R^{d_1\times r}, \mat V \in \R^{d_2\times r}, \norm{\mat U}_F^2 \le c\norm{\mat M^*}_F, \norm{\mat V}_F^2 \le c\norm{\mat M^*}_F \}$, the function $f$ is $\left( (6c+2) \norm{\mat M^*}_F \right)$-smooth.
\end{lem}
\begin{proof}
	We prove smoothness by giving an upper bound on $\lambda_{\max}(\nabla^2 f(\mat U, \mat V)) $ for any $(\mat U, \mat V) \in \cS$.
	
	For any $(\mat U, \mat V) \in \cS$ and any $ \mat \Delta = \begin{pmatrix}
	\mat \Delta_{\mat U} \\ \mat{\Delta}_{\mat V}
	\end{pmatrix}$ ($\mat \Delta_{\mat U} \in \R^{d_1\times r}, \mat \Delta_{\mat V} \in \R^{d_2\times r}$),
	from \eqref{eqn:mf-hessian} we have
	\begin{align*}
	& [\nabla^2 f(\mat U, \mat V)] (\mat \Delta, \mat \Delta) \\
	\le\,& 2 \norm{ \mat U \mat V^\top - \mat M^*}_F \norm{ \mat\Delta_{\mat U} \mat{\Delta}_{\mat V}^\top}_F  + \norm{\mat U \mat\Delta_{\mat V}^\top + \mat\Delta_{\mat U} \mat V^\top }_F^2 \\
	\le\,& 2 \left( \norm{ \mat U}_F \norm{\mat V^\top}_F + \norm{\mat M^*}_F \right) \norm{ \mat\Delta_{\mat U}}_F \norm{\mat{\Delta}_{\mat V}^\top}_F  + \left( \norm{\mat U}_F \norm {\mat\Delta_{\mat V}^\top }_F + \norm{\mat\Delta_{\mat U}}_F \norm{\mat V^\top}_F \right)^2 \\
	\le\,& 2\left( c\norm{\mat M^*}_F + \norm{\mat M^*}_F  \right) \norm{\mat{\Delta}}_F^2 + \left( 2\sqrt{c\norm{\mat M^*}_F} \cdot  \norm{\mat{\Delta}}_F \right)^2 \\
	=\,& (6c+2) \norm{\mat M^*}_F  \norm{\mat{\Delta}}_F^2.
	\end{align*}
	This implies $\lambda_{\max}(\nabla^2 f(\mat U, \mat V)) \le (6c+2) \norm{\mat M^*}_F$.
\end{proof}

\subsection{Proof of Lemma~\ref{lem:mf-balance}}

Recall the following three properties we want to prove in Lemma~\ref{lem:mf-balance}, which we call $\cA(t)$, $\cB(t)$ and $\cC(t)$, respectively:
\begin{align*}
	\cA(t):\qquad & \norm{\mat U_t^\top \mat U_t - \mat V_t^\top \mat V_t}_F \le \epsilon, \\
	\cB(t):\qquad & f(\mat U_{t}, \mat V_{t}) \le f(\mat U_{t-1}, \mat V_{t-1}) \le \cdots \le f(\mat U_{0}, \mat V_{0}) \le 2 \norm{\mat M^*}_F^2, \\
	\cC(t):\qquad & \norm{\mat U_{t}}_F^2 \le  5\sqrt{r} \norm{\mat M^*}_F, \norm{\mat V_t}_F^2 \le  5\sqrt{r} \norm{\mat M^*}_F.
\end{align*}

We use induction to prove these statements.
For $t=0$, we can make the Gaussian variance in the initialization sufficiently small such that with high probability we have $$\norm{\mat U_{0}}_F^2 \le  \epsilon,\qquad \norm{\mat V_0}_F^2 \le  \epsilon, \qquad \norm{\mat U_0^\top \mat U_0 - \mat V_0^\top \mat V_0}_F \le \frac\epsilon2.$$
From now on we assume they are all satisfied.
Then $\cA(0)$ is already satisfied, $\cC(0)$ is satisfied because $\epsilon < \norm{\mat M^*}_F$,
and $\cB(0)$ can be verified by $f(\mat U_{0}, \mat V_{0}) = \frac12 \norm{\mat U_0 \mat V_0^\top - \mat M^*}_F^2 \le \norm{\mat U_0 \mat V_0^\top}_F^2 + \norm{\mat M^*}_F^2 \le \norm{\mat U_0}_F^2 \norm{\mat V_0^\top}_F^2 + \norm{\mat M^*}_F^2 \le \epsilon^2+\norm{\mat M^*}_F^2\le 2\norm{\mat M^*}_F^2$.

To prove $\cA(t)$, $\cB(t)$ and $\cC(t)$ for all $t$, we prove the following three claims. Since we have $\cA(0)$, $\cB(0)$ and $\cC(0)$, if the following claims are all true, the proof will be completed by induction.
\begin{enumerate}[(i)]
	\item $\cB(0), \ldots, \cB(t), \cC(0),  \ldots, \cC(t) \implies \cA(t+1)$;
	\item $\cB(0), \ldots, \cB(t), \cC(t) \implies \cB(t+1)$;
	\item $\cA(t), \cB(t) \implies \cC(t)$.
\end{enumerate}

\begin{claim}
	$\cB(0), \ldots, \cB(t), \cC(0),  \ldots, \cC(t) \implies \cA(t+1)$.
\end{claim}
\begin{proof}
	Using the update rule \eqref{eqn:mf-gd-dynamics} we can calculate
	\begin{align*}
	& \mat U_{t+1}^\top \mat U_{t+1} - \mat V_{t+1}^\top \mat V_{t+1} \\ 
	=\, & \left(\mat U_t - \eta_t (\mat U_t \mat V_t^\top - \mat M^*) \mat V_t\right)^\top \left(\mat U_t - \eta_t (\mat U_t \mat V_t^\top - \mat M^*) \mat V_t\right) \\ & - \left(\mat V_t - \eta_t (\mat U_t \mat V_t^\top - \mat M^*)^\top \mat U_t\right)^\top \left(\mat V_t - \eta_t (\mat U_t \mat V_t^\top - \mat M^*)^\top \mat U_t\right) \\
	=\, & \mat U_t^\top \mat U_t - \mat V_t^\top \mat V_t + \eta_t^2 \left( \mat V_t^\top \mat R_t^\top \mat R_t \mat V_t - \mat U_t^\top \mat R_t^\top \mat R_t \mat U_t \right),
	\end{align*}
	where $\mat R_t = \mat U_t \mat V_t^\top - \mat M^*$.
	Then we have
	\begin{equation} \label{eqn:mf-inproof-1}
	\begin{aligned}
	&\norm{\mat U_{t+1}^\top \mat U_{t+1} - \mat V_{t+1}^\top \mat V_{t+1}}_F\\
	\le\,& \norm{\mat U_t^\top \mat U_t - \mat V_t^\top \mat V_t}_F + \eta_t^2 \left( \norm{\mat V_t^\top \mat R_t^\top \mat R_t \mat V_t }_F + \norm{ \mat U_t^\top \mat R_t^\top \mat R_t \mat U_t}_F \right) \\
	\le\,& \norm{\mat U_t^\top \mat U_t - \mat V_t^\top \mat V_t}_F + \eta_t^2 \left( \norm{\mat V_t}_F^2 \norm{\mat R_t}_F^2  + \norm{\mat U_t}_F^2 \norm{\mat R_t}_F^2  \right) \\
	=\,& \norm{\mat U_t^\top \mat U_t - \mat V_t^\top \mat V_t}_F + 2\eta_t^2 \left( \norm{\mat V_t}_F^2   + \norm{\mat U_t}_F^2   \right) f(\mat U_t, \mat V_t) \\
	\le\,& \norm{\mat U_t^\top \mat U_t - \mat V_t^\top \mat V_t}_F + 2\eta_t^2 \cdot 10\sqrt{r} \norm{\mat M^*}_F \cdot 2 \norm{\mat M^*}_F^2,
	\end{aligned}
	\end{equation}
	where the last line is due to $\cB(t)$ and $\cC(t)$.
	
	Since we have $\cB(t')$ and $\cC(t')$ for all $t' \le t$, \eqref{eqn:mf-inproof-1} is still true when substituting $t$ with any $t'\le t$. Summing all of them and noting $\norm{\mat U_0^\top \mat U_0 - \mat V_0^\top \mat V_0}_F \le \frac{\epsilon}{2}$, we get
	\begin{align*}
	&\norm{\mat U_{t+1}^\top \mat U_{t+1} - \mat V_{t+1}^\top \mat V_{t+1}}_F \\
	\le\,& \norm{\mat U_0^\top \mat U_0 - \mat V_0^\top \mat V_0}_F + 40\sqrt{r} \norm{\mat M^*}_F^3 \sum_{i=0}^t \eta_i^2 \\
	\le\,& \frac\epsilon2 + 40\sqrt{r} \norm{\mat M^*}_F^3 \sum_{i=0}^t \frac{1}{(i+1)^2} \cdot \frac{\epsilon/r}{100^2\norm{\mat M^*}_F^3}\\ 
	\le\,&\epsilon.
	\end{align*}
	Therefore we have proved $\cA(t+1)$.
\end{proof}

\begin{claim}
	$\cB(0), \ldots, \cB(t), \cC(t) \implies \cB(t+1)$.
\end{claim}
\begin{proof}
	Note that we only need to show $f(\mat U_{t+1}, \mat V_{t+1}) \le f(\mat U_{t}, \mat V_{t})$.
	We prove this using the standard analysis of gradient descent, for which we need the smoothness of the objective function $f$ (Lemma~\ref{lem:mf-local-smooth}).
	We first need to bound $\norm{\mat U_t}_F$, $\norm{\mat V_t}_F$, $\norm{\mat U_{t+1}}_F$ and $\norm{\mat V_{t+1}}_F$. We know from $\cC(t)$ that $\norm{\mat U_{t}}_F^2 \le  5\sqrt{r} \norm{\mat M^*}_F$ and $\norm{\mat V_t}_F^2 \le  5\sqrt{r} \norm{\mat M^*}_F$.
	We can also bound $\norm{\mat U_{t+1}}_F^2$ and $\norm{\mat V_{t+1}}_F^2$ easily from the GD update rule:
	\begin{align*}
		&\norm{\mat U_{t+1}}_F^2 \\
		=\,& \norm{\mat U_t - \eta_t (\mat U_t \mat V_t^\top - \mat M^*) \mat V_t}_F^2 \\
		\le\,& 2 \norm{\mat U_t}_F^2 + 2\eta_t^2 \norm{\mat U_t \mat V_t^\top - \mat M^*}_F^2 \norm{\mat V_t}_F^2 \\
		\le\,& 2\cdot 5\sqrt{r} \norm{\mat M^*}_F + 2 \eta_t^2 \cdot 2f(\mat U_t, \mat V_t) \cdot 5\sqrt{r} \norm{\mat M^*}_F\\
		\le\,& 10\sqrt{r} \norm{\mat M^*}_F + 2 \cdot \frac{\epsilon/r}{100^2 (t+1)^2 \norm{\mat M^*}_F^3} \cdot 4\norm{\mat M^*}_F^2 \cdot 5\sqrt{r} \norm{\mat M^*}_F & \text{(using $\cB(t)$)} \\
		\le \,& 10\sqrt{r} \norm{\mat M^*}_F + \frac{\epsilon}{100} \\
		\le\, & 11\sqrt{r} \norm{\mat M^*}_F. & \text{(using $\epsilon < \norm{\mat M^*}_F$)}
	\end{align*}
	Let $\beta =  (66\sqrt{r}+2) \norm{\mat M^*}_F $.
	 From Lemma~\ref{lem:mf-local-smooth}, $f$ is $\beta$-smooth over $\cS = \{(\mat U, \mat V): \norm{\mat U}_F^2 \le 11\sqrt{r} \norm{\mat M^*}_F, \norm{\mat V}_F^2 \le 11\sqrt{r} \norm{\mat M^*}_F \}$.
	 Also note that $\eta_t < \frac{1}{\beta}$ by our choice. Then using smoothness we have
	 \begin{equation} \label{eqn:mf-obj-decrease}
	 \begin{aligned}
	 &f(\mat U_{t+1}, \mat V_{t+1}) \\
	 \le\,& f(\mat U_{t}, \mat V_{t}) + \left\langle  \nabla f(\mat U_t, \mat V_t), \begin{pmatrix}
		\mat U_{t+1}\\ \mat V_{t+1}
	 	\end{pmatrix} - \begin{pmatrix}
	 	\mat U_{t}\\ \mat V_{t}
	 	\end{pmatrix} \right\rangle + \frac{\beta}{2} \norm{ \begin{pmatrix}
	 		\mat U_{t+1}\\ \mat V_{t+1}
	 		\end{pmatrix} - \begin{pmatrix}
	 		\mat U_{t}\\ \mat V_{t}
	 		\end{pmatrix}  }_F^2 \\
 	=\,& f(\mat U_{t}, \mat V_{t}) - \eta_t \norm{ \nabla f(\mat U_t, \mat V_t) }_F^2 + \frac{\beta}{2}\eta_t^2 \norm{ \nabla f(\mat U_t, \mat V_t) }_F^2 \\
 	\le\,& f(\mat U_{t}, \mat V_{t}) - \frac{\eta_t}{2} \norm{ \nabla f(\mat U_t, \mat V_t) }_F^2.
	 \end{aligned}
	 \end{equation}
	 Therefore we have shown $\cB(t+1)$.
\end{proof}

\begin{claim}
	$\cA(t), \cB(t) \implies \cC(t)$.
\end{claim}
\begin{proof}
	From $\cB(t)$ we know $\frac12 \norm{\mat U_t \mat V_t^\top - \mat M^*}_F^2 \le 2 \norm{\mat M^*}_F^2$ which implies $\norm{\mat U_t \mat V_t^\top}_F \le 3\norm{\mat M^*}_F$.
	Therefore it suffices to prove
	\begin{equation} \label{eqn:mf-inproof-toshow-1}
		 \norm{\mat U \mat V^\top}_F \le 3\norm{\mat M^*}_F,
		 \norm{\mat U^\top \mat U - \mat V^\top \mat V}_F \le \epsilon
		 \implies \norm{\mat U}_F^2 \le 5\sqrt{r} \norm{\mat M^*}_F, \norm{\mat V}_F^2 \le 5\sqrt{r} \norm{\mat M^*}_F.
	\end{equation}
	
	Now we prove \eqref{eqn:mf-inproof-toshow-1}.
	Consider the SVD $\mat U = \mat\Phi \mat\Sigma \mat{\Psi}^\top$, where $\mat \Phi \in \R^{d_1\times d_1}$ and $\mat{\Psi} \in \R^{r\times r}$ are orthogonal matrices, and $\mat\Sigma \in \R^{d_1\times r}$ is a diagonal matrix.
	Let $\sigma_i = \mat \Sigma[i, i]$ ($i\in[r]$) which are all the singular values of $\mat U$.
	Define $\widetilde{\mat V} = \mat V \mat\Psi$. Then we have
	\begin{align*}
	3\norm{\mat M^*}_F
	\ge \norm{\mat U \mat V^\top}_F 
	= \norm{\mat\Phi \mat\Sigma \mat{\Psi}^\top  \mat{\Psi} \widetilde{\mat V}^\top }_F
	= 	\norm{\mat \Sigma \widetilde{\mat V}^\top }_F
	= \sqrt{\sum_{i=1}^r \sigma_i^2 \norm{\widetilde{\mat V}[:,i]}^2}
	\end{align*}
	and
	\begin{align*}
	\epsilon &\ge  \norm{\mat U^\top \mat U - \mat V^\top \mat V}_F
	= \norm{ \mat\Psi \mat\Sigma^\top \mat{\Phi}^\top \mat\Phi \mat\Sigma \mat{\Psi}^\top - \mat{\Psi} \widetilde{\mat V}^\top \widetilde{\mat V} \mat{\Psi}^\top }_F
	= \norm{ \mat\Sigma^\top \mat\Sigma  -  \widetilde{\mat V}^\top \widetilde{\mat V} }_F \\
	&\ge \sqrt{\sum_{i=1}^r \left( \sigma_i^2  - \norm{\widetilde{\mat V}[:,i]}^2 \right)^2}.
	\end{align*}
	Using the above two inequalities we get
	\begin{align*}
	\sum_{i=1}^r \sigma_i^4
	&\le \sum_{i=1}^r \left( \sigma_i^4 + \norm{\widetilde{\mat V}[:,i]}^4 \right)
	= \sum_{i=1}^r \left( \sigma_i^2  - \norm{\widetilde{\mat V}[:,i]}^2 \right)^2 + 2 \sum_{i=1}^r \sigma_i^2 \norm{\widetilde{\mat V}[:,i]}^2 \\
	&\le \epsilon^2 + 2 \left(3\norm{\mat M^*}_F\right)^2
	\le 19 \norm{\mat M^*}_F^2.
	\end{align*}
	Then by the Cauchy-Schwarz inequality we have
	\begin{align*}
	\norm{\mat U}_F^2 = \sum_{i=1}^r \sigma_i^2
	\le \sqrt{r \sum_{i=1}^r \sigma_i^4}
	\le \sqrt{r\cdot 19 \norm{\mat M^*}_F^2}
	\le 5\sqrt{r} \norm{\mat M^*}_F.
	\end{align*}
	Similarly, we also have $\norm{\mat V}_F^2 \le 5\sqrt{r} \norm{\mat M^*}_F$. Therefore we have proved \eqref{eqn:mf-inproof-toshow-1}.
\end{proof}

\subsection{Convergence to a Stationary Point}

With the balancedness and boundedness properties in Lemma~\ref{lem:mf-balance}, it is then standard to show that $(\mat U_t, \mat V_t)$ converges to a stationary point of $f$.

\begin{lem} \label{lem:mf-convergence}
	Under the setting of Theorem~\ref{thm:mf-main}, with high probability $\lim_{t\to\infty}(\mat U_t, \mat V_t) = (\bar{\mat U}, \bar{\mat V})$ exists, and $(\bar{\mat U}, \bar{\mat V})$ is a stationary point of $f$.
	Furthermore, $(\bar{\mat U}, \bar{\mat V})$ satisfies $\norm{\bar{\mat U}^\top \bar{\mat U} - \bar{\mat V}^\top \bar{\mat V}}\le \epsilon$.
\end{lem}
\begin{proof}
	We assume the three properties in Lemma~\ref{lem:mf-balance} hold, which happens with high probability.
	Then from \eqref{eqn:mf-obj-decrease} we have
	\begin{equation}\label{eqn:mf-obj-decrease-2}
	\begin{aligned}
	f(\mat U_{t+1}, \mat V_{t+1})
	&\le f(\mat U_{t}, \mat V_{t}) - \frac{\eta_t}{2} \norm{ \nabla f(\mat U_t, \mat V_t) }_F^2 \\
	&= f(\mat U_{t}, \mat V_{t}) - \frac{1}{2} \norm{ \nabla f(\mat U_t, \mat V_t) }_F \norm{\begin{pmatrix}
		\mat U_{t+1} \\ \mat V_{t+1}
		\end{pmatrix} - \begin{pmatrix}
		\mat U_{t} \\ \mat V_{t}
		\end{pmatrix}}_F .
	\end{aligned}
	\end{equation}
	Under the above descent condition, the result of \cite{absil2005convergence} says that the iterates either diverge to infinity or converge to a fixed point.
	According to Lemma~\ref{lem:mf-balance}, \{$(\mat U_t, \mat V_t)\}_{t=1}^\infty$ are all bounded, so they have to converge to a fixed point $(\bar{\mat U}, \bar{\mat V})$ as $t\to\infty$.
	
	Next, from \eqref{eqn:mf-obj-decrease-2} we know that $\sum_{t=1}^\infty \frac{\eta_t}{2} \norm{ \nabla f(\mat U_t, \mat V_t) }_F^2 \le f(\mat U_0, \mat V_0)$ is bounded.
	Notice that $\eta_t$ scales like $1/t$. So we must have $\liminf_{t\to\infty} \norm{ \nabla f(\mat U_t, \mat V_t) }_F = 0$.
	Then according to the smoothness of $f$ in a bounded region (Lemma~\ref{lem:mf-local-smooth}) we conclude $ \nabla f(\bar{\mat U}, \bar{\mat V})  = \mat0$, i.e., $(\bar{\mat U}, \bar{\mat V})$ is a stationary point.
	
	The second part of the lemma is evident according to Lemma~\ref{lem:mf-balance} (i).
\end{proof}

\subsection{Proof of Lemma~\ref{lem:mf-strict-saddle}}

The main idea in the proof is similar to \cite{ge2017no}. We want to find a direction $\mat \Delta$ such that either $[\nabla^2 f(\mat U, \mat V)] (\mat \Delta, \mat \Delta)$ is negative or $(\mat U, \mat V)$ is close to a global minimum. We show that this is possible when $\norm{\mat U^\top \mat U - \mat V^\top \mat V}_F \le \epsilon$.

First we define some notation.
Take the SVD $\mat M^* = \mat \Phi^* \mat \Sigma^* \mat \Psi^{*\top}$, where $\mat \Phi^* \in \R^{d_1\times r}$ and $\mat \Psi^* \in \R^{d_2\times r}$ have orthonormal columns and $\mat \Sigma^* \in \R^{r\times r}$ is diagonal.
Denote $\mat U^* = \mat \Phi^*  (\mat \Sigma^*)^{1/2}$ and $\mat V^* = \mat \Psi^*  (\mat \Sigma^*)^{1/2}$.
Then we have $\mat U^* \mat V^{*\top} = \mat M^*$ (i.e., $(\mat U^*, \mat V^*)$ is a global minimum) and $\mat U^{*\top} \mat U^* = \mat V^{*\top} \mat V^*$.

Let $\mat M = \mat U \mat V^\top$, $\mat W = \begin{pmatrix}
\mat U\\ \mat V
\end{pmatrix}$ and $\mat W^* = \begin{pmatrix}
\mat U^*\\ \mat V^*
\end{pmatrix}$.
Define $$\mat{R} = \argmin_{\mat R' \in \R^{r\times r} \text{, orthogonal}} \norm{\mat W - \mat W^* \mat R'}_F$$
and
\[
\mat\Delta = \mat W - \mat W^* \mat R.
\]

We will show that $\mat\Delta$ is the desired direction.
Recall \eqref{eqn:mf-hessian}:
\begin{equation} \label{eqn:mf-hessian-to-bound}
\begin{aligned}  \
[\nabla^2 f(\mat U, \mat V)] (\mat \Delta, \mat \Delta)
= 2 \left\langle \mat M - \mat M^*, \mat\Delta_{\mat U} \mat{\Delta}_{\mat V}^\top \right\rangle + \norm{\mat U \mat\Delta_{\mat V}^\top  + \mat\Delta_{\mat U} \mat V^\top }_F^2, 
\end{aligned}
\end{equation}
where $ \mat \Delta = \begin{pmatrix}
\mat \Delta_{\mat U} \\ \mat{\Delta}_{\mat V}
\end{pmatrix}, 
\mat \Delta_{\mat U} \in \R^{d_1\times r}, \mat \Delta_{\mat V} \in \R^{d_2\times r}$.
We consider the two terms in \eqref{eqn:mf-hessian-to-bound} separately.

For the first term in \eqref{eqn:mf-hessian-to-bound}, we have:
\begin{claim} \label{claim:mf-hessian-to-bound-1}
	$\left\langle \mat M - \mat M^*, \mat\Delta_{\mat U} \mat{\Delta}_{\mat V}^\top \right\rangle  = -\norm{\mat M - \mat M^*}_F^2$.
\end{claim}
\begin{proof}
	Since $(\mat U, \mat V)$ is a stationary point of $f$, we have the first-order optimality condition:
	\begin{equation} \label{eqn:mf-inproof-first-order-opt-cond}
	\frac{\partial f(\mat U, \mat V)}{\partial \mat U} = (\mat M - \mat M^*) \mat V = \mat0, \qquad
	\frac{\partial f(\mat U, \mat V)}{\partial \mat V} = (\mat M - \mat M^*)^\top \mat U = \mat0.
	\end{equation}

Note that $\mat \Delta_{\mat U} = \mat U - \mat U^* \mat R$ and $\mat \Delta_{\mat V} = \mat V - \mat V^* \mat R$.
We have
\begin{align*}
&\left\langle \mat M - \mat M^*, \mat\Delta_{\mat U} \mat{\Delta}_{\mat V}^\top \right\rangle \\
=\,& \left\langle \mat M - \mat M^*, (\mat U - \mat U^* \mat R) (\mat V - \mat V^* \mat R)^\top \right\rangle \\
=\,& \left\langle \mat M - \mat M^*, \mat M - \mat U^* \mat R \mat V^\top - \mat U \mat R^\top \mat V^{*\top} + \mat M^*  \right\rangle \\
=\,& \left\langle \mat M - \mat M^*, \mat M^*\right\rangle \\
=\,& \left\langle \mat M - \mat M^*, \mat M^* - \mat M \right\rangle \\
=\,& -\norm{\mat M - \mat M^*}_F^2,
\end{align*}
where we have used the following consequences of \eqref{eqn:mf-inproof-first-order-opt-cond}:
\begin{align*}
&\left\langle \mat M - \mat M^*, \mat M \right\rangle = \left\langle \mat M - \mat M^*, \mat U \mat V^\top \right\rangle = 0, \\
&\left\langle \mat M - \mat M^*, \mat U^* \mat R \mat V^\top \right\rangle  = 0, \\
&\left\langle \mat M - \mat M^*, \mat U \mat R^\top \mat V^{*\top}  \right\rangle  = 0.
\end{align*}
\end{proof}

The second term in \eqref{eqn:mf-hessian-to-bound} has the following upper bound:
\begin{claim} \label{claim:mf-hessian-to-bound-2}
$\norm{\mat U \mat\Delta_{\mat V} + \mat\Delta_{\mat U} \mat V}_F^2 \le \norm{\mat M - \mat M^*}_F^2 + \frac12 \epsilon^2 $.
\end{claim}
\begin{proof}
	We make use of the following identities, all of which can be directly verified by plugging in definitions:
	\begin{equation} \label{eqn:mf-inproof-identity-1}
	\mat U \mat\Delta_{\mat V}^\top + \mat\Delta_{\mat U} \mat V^\top = \mat\Delta_{\mat U} \mat\Delta_{\mat V}^\top + \mat M - \mat M^*,
	\end{equation}
	\begin{equation} \label{eqn:mf-inproof-identity-2}
	\norm{\mat{\Delta} \mat{\Delta}^\top}_F^2 = 4 \norm{\mat\Delta_{\mat U} \mat\Delta_{\mat V}^\top}_F^2 + \norm{\mat\Delta_{\mat U}^\top\mat\Delta_{\mat U} - \mat\Delta_{\mat V}^\top \mat\Delta_{\mat V}}_F^2,
	\end{equation}
	\begin{equation} \label{eqn:mf-inproof-identity-3}
	\begin{aligned}
	\norm{\mat W \mat W^\top - \mat W^* \mat W^{*\top}}_F^2 =\ & 4\norm{\mat M - \mat M^*}_F^2 - 2 \norm{\mat U^\top \mat U^* - \mat V^\top \mat V^*}_F^2 \\&+ \norm{\mat U^\top \mat U - \mat V^\top \mat V}_F^2 + \norm{\mat U^{*\top} \mat U^* - \mat V^{*\top} \mat V^*}_F^2.
	\end{aligned}
	\end{equation}
	We also need the following inequality, which is \citep[Lemma 6]{ge2017no}:
	\begin{equation} \label{eqn:mf-inproof-rong-ineq}
	\norm{\mat{\Delta} \mat{\Delta}^\top}_F^2 \le 2 \norm{\mat W \mat W^\top - \mat W^* \mat W^{*\top}}_F^2.
	\end{equation}
	
	Now we can prove the desired bound as follows:
	\begin{align*}
	& \norm{\mat U \mat\Delta_{\mat V} + \mat\Delta_{\mat U} \mat V}_F^2 \\
	=\, & \norm{\mat\Delta_{\mat U} \mat\Delta_{\mat V}^\top + \mat M - \mat M^*}_F^2 & (\eqref{eqn:mf-inproof-identity-1}) \\
	=\,& \norm{\mat\Delta_{\mat U} \mat\Delta_{\mat V}^\top }_F^2 + 2 \left\langle \mat M - \mat M^*, \mat\Delta_{\mat U} \mat{\Delta}_{\mat V}^\top \right\rangle  + \norm{\mat M - \mat M^*}_F^2 \\	
	=\,& \norm{\mat\Delta_{\mat U} \mat\Delta_{\mat V}^\top }_F^2 - \norm{\mat M - \mat M^*}_F^2 & (\text{Claim~\ref{claim:mf-hessian-to-bound-1}})\\
	\le\, & \frac14 \norm{\mat{\Delta} \mat{\Delta}^\top}_F^2 - \norm{\mat M - \mat M^*}_F^2  & (\eqref{eqn:mf-inproof-identity-2}) \\
	\le\, & \frac12 \norm{\mat W \mat W^\top - \mat W^* \mat W^{*\top}}_F^2  - \norm{\mat M - \mat M^*}_F^2  & (\eqref{eqn:mf-inproof-rong-ineq}) \\
	=\, & 2\norm{\mat M - \mat M^*}_F^2 -  \norm{\mat U^\top \mat U^* - \mat V^\top \mat V^*}_F^2  + \frac12\norm{\mat U^\top \mat U - \mat V^\top \mat V}_F^2 \\& + \frac12\norm{\mat U^{*\top} \mat U^* - \mat V^{*\top} \mat V^*}_F^2 - \norm{\mat M - \mat M^*}_F^2 & (\eqref{eqn:mf-inproof-identity-3}) \\
	\le\,& \norm{\mat M - \mat M^*}_F^2 + \frac12 \epsilon^2,
	\end{align*}
	where in the last line we have used $\mat U^{*\top} \mat U^* = \mat V^{*\top} \mat V^*$ and $\norm{\mat U^{\top} \mat U - \mat V^{\top} \mat V} \le \epsilon$.
\end{proof}

Using Claims~\ref{claim:mf-hessian-to-bound-1} and \ref{claim:mf-hessian-to-bound-2}, we obtain an upper bound on \eqref{eqn:mf-hessian-to-bound}:
\begin{align*}
[\nabla^2 f(\mat U, \mat V)] (\mat \Delta, \mat \Delta) \le -\norm{\mat M - \mat M^*}_F^2 + \frac12 \epsilon^2.
\end{align*}
Therefore, we have either $\norm{\mat U \mat V^\top - \mat M^*}_F = \norm{\mat M - \mat M^*}_F \le \epsilon$ or $[\nabla^2 f(\mat U, \mat V)] (\mat \Delta, \mat \Delta) \le -\frac12 \epsilon^2 <0$.
In the latter case, $(\mat U, \mat V)$ is a strict saddle point of $f$.
This completes the proof of Lemma~\ref{lem:mf-strict-saddle}.

\subsection{Finishing the Proof of Theorem~\ref{thm:mf-main}}

Theorem~\ref{thm:mf-main} is a direct corollary of Lemma~\ref{lem:mf-convergence}, Lemma~\ref{lem:mf-strict-saddle}, and the fact that gradient descent does not converge to a strict saddle point almost surely \citep{lee2016gradient, panageas2016gradient}.

\section{Proof for Rank-$1$ Matrix Factorization (Theorem~\ref{thm:rank_1})}
\label{sec:proof-rank-1}
In this section we prove Theorem~\ref{thm:rank_1}.

\begin{proof}[Proof of Theorem~\ref{thm:rank_1}]
We define the following four key quantities:
\begin{align*}
\alpha_t = \vect{u}_t^\top\vect{u}^*,\qquad
\alpha_{t,\perp} = \norm{\mat{U}^{*}_\perp \vect{u}_t}_2, \qquad
\beta_{t} = \vect{v}_t^\top \vect{v}^*, \qquad
\beta_{t,\perp} = \norm{\mat{V}^{*}_\perp \vect{v}_t}_2,
\end{align*}
where $\mat{U}_\perp^{*} = \mat I - \vect{u}^*\vect{u}^{*\top}$ and $\mat{V}_\perp^{*} = \mat I - \vect{v}^*\vect{v}^{*\top}$ are the projection matrices onto the orthogonal complement spaces of $\vect{u}^*$ and $\vect{v}^*$, respectively.
Notice that $\norm{\vect{u}_t}_2^2 = \alpha_t^2 + \alpha_{t,\perp}^2$ and $\norm{\vect{v}_t}_2^2 = \beta_t^2 + \beta_{t,\perp}^2$.
It turns out that we can write down the explicit formulas for the dynamics of these quantities:
\begin{equation} \label{eqn:rank-1-alpha-beta-dynamics}
\begin{aligned}
\alpha_{t+1} = 
\left(1-\eta\left(\beta_t^2 +\beta_{t,\perp}^2\right)\right)\alpha_t + \eta\sigma_1\beta_t, \qquad
&\beta_{t+1} = \left(1-\eta\left(\alpha_t^2+\alpha_{t,\perp}^2\right)\right)\beta_t + \eta_1\sigma_1 \alpha_t,\\
\alpha_{t+1,\perp} = \left(1-\eta\left(\beta_t^2+\beta_{t,\perp}^2\right)\right)\alpha_{t,\perp}, \qquad
&\beta_{t+1,\perp} = \left(1-\eta\left(\alpha_t^2+\alpha_{t,\perp}^2\right)\right)\beta_{t,\perp}.
\end{aligned}
\end{equation}

To facilitate the analysis, we also define:
\begin{align*}
%r_t = &\alpha_t\beta_t\\
h_t = &\alpha_t\beta_t - \sigma_1,\\
\xi_t = &\alpha_{t,\perp}^2+\beta_{t,\perp}^2.
\end{align*}
%Denote $\xi_t = \alpha_{t,\perp}^2+\beta_{t,\perp}^2$.
Then our goal is to show $\xi_t\to0$ and
%$\beta_{t,\perp} = 0, \alpha_{t,\perp}=0$ and
$h_t\to0$ as $t\to\infty$.
We calculate the dynamics of $h_t$ and $\xi_t$:
%\begin{align}
%r_{t+1} = &\left(1-\eta\left(\alpha_t^2+\beta_t^2+\alpha_{t,\perp}^2+\beta_{t,\perp}^2\right)+\eta^2\left[
%\left(\beta_t^2+\beta_{t,\perp}^2\right)\left(\alpha_{t}^2+\alpha_{t,\perp}^2\right)+\sigma_1^2
%\right]\right)r_t \nonumber\\
%& + \eta \sigma_1\left(1-\eta\left(\beta_t^2+\beta_{t,\perp}^2\right)\right)\alpha_t^2 + \eta\sigma_1\left(1-\eta\left(\alpha_t^2+\alpha_{t,\perp}^2\right)\right)\beta_t^2. \label{eqn:dynamics_r}
%\end{align}
\begin{equation} \label{eqn:rank-1-h-xi-dynamics}
\begin{aligned}
h_{t+1} = &\left(1-\eta\left(\alpha_t^2+\beta_t^2\right)
+\eta^2\left(\alpha_t \beta_t h_t+\alpha_t^2\beta_{t,\perp}^2+\beta_t^2\alpha_{t,\perp}^2 + \alpha_{t,\perp}^2\beta_{t,\perp}^2\right)
\right)h_t 
-\eta \alpha_t \beta_t \xi_t  + \eta^2\sigma_1\alpha_{t,\perp}^2\beta_{t,\perp}^2,\\ %\label{eqn:dynamics_h}\\
\xi_{t+1} = & \left(1-\eta\left(\beta_t^2+\beta_{t,\perp}^2\right)\right)^2 \alpha_{t,\perp}^2 + \left(1-\eta\left(\alpha_t^2+\alpha_{t,\perp}^2\right)\right)^2 \beta_{t,\perp}^2. %\label{eqn:dynamics_xi}
\end{aligned}
\end{equation}

According to our initialization scheme, with high probability we have %by standard random matrix theory, with overwhelming probability, we know the following facts about these quantities
$
|\alpha_0| , |\beta_0| \in \left[ 0.1 c_{init}\sqrt{\frac{\sigma_1}{d}}, 10 c_{init}\sqrt{\frac{\sigma_1}{d}} \right]$ and $|\alpha_{0,\perp}|, |\beta_{0,\perp}| \le 10c_{init} \sqrt{\sigma_1}. %,  \beta_0 \asymp\frac{c_{init}\sqrt{\sigma_1}}{\sqrt{d}}, \beta_{0,\perp} \lesssim c_{init} \sqrt{\sigma_1}.
$
We assume that these conditions are satisfied.
We also assume that the signal at the beginning is positive: $\alpha_0 \beta_0 > 0$, which holds with probability $1/2$.
Without loss of generality we assume $\alpha_0, \beta_0>0$.\footnote{If $\alpha_0, \beta_0<0$, we can simply flip the signs of $\vect u^*$ and $\vect v^*$.}
% using Gaussian initialization.

%To prove %$c_0\le \frac{\abs{\vect{u}_t^\top\vect{u}^*}}{\abs{\vect{v}^\top \vect{y}_t}}\le C_0$ 
%$c_0 \le \abs{\alpha_t} / \abs{\beta_t} \le C_0$
%for all iterations, we cannot only analyze quantity $\frac{\abs{\vect{u^\top\vect{x}_t}}}{\abs{\vect{v}^\top \vect{y}_t}}$ along, the convergence rate to the global minimum also plays an important role in affecting $\frac{\abs{\vect{u^\top\vect{x}_t}}}{\abs{\vect{v}^\top \vect{y}_t}}$.
We divide the dynamics into two stages.
\begin{lem}[Stage 1: escaping from saddle point $(\vect{
		0},\vect{0})$]\label{lem:rank_1_stage_1}
	Let $T_1 = \min \left\{t\in\N: \alpha_t^2 + \beta_t^2 \ge \frac12 \sigma_1\right\}$.
	%If $\eta \le \frac{c_{step}}{\sigma_1}$, 
	Then for $t=0, 1,\ldots,T_1-1$, the followings hold:
	\begin{enumerate}[(i)]
		\item Positive signal strengths: $\alpha_t, \beta_t>0$;
		\item Small magnitudes in complement space: $\xi_t \le \xi_0 \le 100 c_{init}^2 \sigma_1$;
		\item Growth of magnitude in signal space: $\left(1+\frac{c_{step}}{3}\right) (\alpha_t+\beta_t)\le {\alpha_{t+1}+\beta_{t+1}} \le \left(1+c_{step}\right) (\alpha_t+\beta_t)$;
		\item Bounded ratio between two layers: $\abs{\alpha_t - \beta_t} \le \frac{99}{101}(\alpha_t+\beta_t)$. % where $c_{diff} = \frac{\abs{\alpha_0-\beta_0}}{\alpha_0+\beta_0} \in [0, \frac{99}{101}]$.
	\end{enumerate}
	Furthermore, we have $T_1 = O(\log d)$.
	% $T_1 = O\left(\frac{1}{\eta\sigma_1}\log\left(\frac{\sigma_1}{\alpha_0+\beta_0}\right)\right)$ iterations, $T_1 \le T_0$, we have $\alpha_{T_1} \beta_{T_1} \ge c \sigma_1$ for some small constant $0 < c < 1$.
\end{lem}
In this stage, the strengths in the complement spaces remain small ($\xi_t \le \xi_0$) and the strength in the signal space is growing exponentially (${\alpha_{t+1}+\beta_{t+1}} \ge \left(1+\frac{c_{step}}{3}\right) ({\alpha_t+\beta_t})$).
Furthermore,  $\abs{\alpha_t - \beta_t} \le \frac{99}{101}({\alpha_t+\beta_t})$ implies $\frac{\alpha_t}{\beta_t} \in [\frac{1}{100}, 100]$, which means the signal strengths in the two layers are of the same order. %$\frac{1-c_{diff}}{1+c_{diff}}\le \frac{\alpha_t}{\beta_t} \le \frac{1+c_{diff}}{1-c_{diff}}$, which is our desired result.
%Here we consider $\abs{\alpha_t - \beta_t} / \abs{\alpha_t + \beta_t}$ instead of directly bounding $\frac{\alpha_t}{\beta_t}$ because the dynamics of $\abs{\alpha_t - \beta_t} / \abs{\alpha_t - \beta_t}$ is much simpler to analyze.
%This will be apparent in the proof.

Then we enter stage 2, which is essentially a local convergence phase.
The following lemma characterizes the behaviors of the strengths in the signal and noise spaces in this stage.
\begin{lem}[Stage 2: convergence to global minimum]\label{lem:rank_1_stage_2}
	%Suppose at $T_1$-th iteration, the followings hold \begin{align*}
	%\alpha_{T_1}\beta_{T_1} \ge c\sigma_1, \qquad
	%\alpha_{T_1,\perp}^2+\beta_{T_1,\perp}^2 \le &10c_{init} \sigma_1,\qquad
	%\abs{\alpha_{T_1}-\beta_{T_1}} \le c_{diff}\abs{\alpha_{T_1}+\beta_{T_1}},
	%\end{align*} where $c > 1/5$, $c_{init} < 1/100$, $c_{diff}< 1$ and $\eta \le \frac{c_{step}}{\sigma_1}$ for $c_{step} \le 1/100$.
	Let $T_1$ be as defined in Lemma~\ref{lem:rank_1_stage_1}.
	Then there exists a universal constant $c_1>0$ such that the followings hold for all $t \ge T_1$:
	\begin{enumerate}[(a)]
		\item  Non-vanishing signal strengths in both layers: $\alpha_t, \beta_t \ge  \sqrt{c_1\sigma_1}$;
		\item Bounded signal strengths: $\alpha_t \beta_t \le \sigma_1$, i.e., $h_t\le 0$;
		\item Shrinking magnitudes in complement spaces: $\xi_{t} \le (1- c_1 c_{step})^{t-T_1}\xi_{0} \le (1- c_1 c_{step})^{t-T_1}\cdot 100c_{init}^2\sigma_1$;
		\item Convergence in signal space: $|h_{t+1}| \le (1-c_1 c_{step}) |h_t| + c_{step} \xi_t$.
	\end{enumerate}
	%Let $\tau_t = \min\{\alpha_t,\beta_t\}$ and $\tau_{min}  \triangleq \frac{\tau_{T_1}}{2}$.
	%Then for all $t = T+1,T+2,\ldots$, the followings hold 
%	\begin{align*}
%	\tau_t \ge \prod_{i=T+1}^{t}\left(1-\eta\xi_0\left(1-\eta\tau_{min}^2\right)^{i-T}\right)\tau_T, \qquad
%	\xi_t \le \left(1-\eta\tau_{min}^2\right)^{t-T} \xi_0, \qquad
%	\alpha_t\beta_t \le \sigma_1.
%	\end{align*}
\end{lem}
%In Lemma~\ref{lem:rank_1_stage_2}, $\tau_t$ characterizes the smaller signal strength in $\vect{x}_t$ and $\vect{y}_t$ and $\xi_t$ characterizes the strength in the noise space.
%Note even though $\tau_t$ may decrease, with some algebra, Lemma~\ref{lem:rank_1_stage_2} shows $\tau_t$ is uniformly bounded by $\Omega\left(\sqrt{\sigma_1}\right)$ for all $t \ge T_1$. 
%Combining with the fact that $\alpha_t\beta_t \le \sigma_1$, we know $c_0\le \frac{\alpha_t}{\beta_t} \le C_0$ for all $t \ge T_1$.
%Similar to Lemma~\ref{lem:rank_1_stage_1}, here we consider $\tau_t$ and $\alpha_t\beta_t$ instead of directly analyzing $\frac{\alpha_t}{\beta_t}$ because dynamics of $\tau_t$ and $\alpha_t\beta_t$ are easier to analyze.

Note that properties (a) and (b) in Lemma~\ref{lem:rank_1_stage_2} imply $c_0 \le \frac{\alpha_t}{\beta_t} \le C_0$ for all $t\ge T_1$, where $c_0, C_0>0$ are universal constants.
Property (c) implies that for all $t\ge T_1+T_2$ where $T_2 = \Theta(\log \frac1\epsilon)$, we have $\xi_t = O(\epsilon \sigma_1)$.
Then property (d) tells us that after another $T_3 = \Theta(\log \frac1\epsilon)$ iterations, we can ensure $|h_t| = O(\epsilon \sigma_1)$ for all $t\ge T_1+T_2+T_3$.
These imply $\norm{\vect{u}_t \vect{v}_t^\top - \mat M^*}_F = O(\epsilon \sigma_1)$ after $t = T_1+T_2+T_3= O(\log \frac{d}{\epsilon})$ iterations, completing the proof of Theorem~\ref{thm:rank_1}.
%A consequence is the following local convergence rate result, which is proved by examining the dynamics of $\alpha_t\beta_t$.
%\begin{lem}\label{lem:rank_1local_convergence_rate}
%	Under the same assumptions as in Lemma~\ref{lem:rank_1_stage_2}, we have after $T_2 = O\left(\frac{1}{\eta \sigma_1}\log\left(\frac{1}{\epsilon}\right)\right)$ iterations, $
%	\norm{\vect{u}\vect{v}^\top - \mat{M}^*}_F^2 \le \epsilon^2 \sigma_1^2.$
%\end{lem}
%Lemmas \ref{lem:rank_1_stage_1}, \ref{lem:rank_1_stage_2} and \ref{lem:rank_1local_convergence_rate} together imply Theorem~\ref{thm:rank_1}.
%The proofs of these lemmas are given in Appendix~\ref{sec:proof-rank-1}.
\end{proof}

Now we prove Lemmas~\ref{lem:rank_1_stage_1} and \ref{lem:rank_1_stage_2}.

%\begin{lem}\label{lem:bounded_growth}
%	For any $x<1$, $\Pi_{i=1}^\infty\left(1+bx^i\right) \le \exp\left(\frac{b}{1-x}\right)$.
%\end{lem}

%To prove Lemma~\ref{lem:rank_1_stage_1}.
%We begin with the following lemma which bounds  $\norm{\vect{u}_t}_2^2 + \norm{\vect{v}_t}_2^2$.
%
%\begin{lem}\label{lem:norm_upper_bound}
%Suppose the statements in Lemma~\ref{lem:rank_1_stage_1} are true for $1,2, \ldots, t$. Then we have $\norm{\vect{u}_t}_2^2 + \norm{\vect{v}_t}_2^2 \le C_1\sigma_1$ for some constant $C_1 > 0$.
%\end{lem}
%\begin{proof}
%By induction hypothesis, for any $t \le T$, $\xi_t \le \xi_0 \le Cc_{init}^2\sigma_1$ for some constant $C > 0$.
%Further, since $\abs{\alpha_t-\beta_t}\le c_{diff} \abs{\alpha_t+\beta_t}$, we have \begin{align*}
%\frac{1-c_{diff}}{1+c_{diff}}\le \frac{\abs{\alpha_t}}{\abs{\beta_t}} \le \frac{1+c_{diff}}{1-c_{diff}}.
%\end{align*}
%Also recall $\alpha_t\beta_t \le \sigma_1$ for $t \le T_1$, so we can bound $\max\left\{\alpha_t^2+\beta_t^2\right\}\le \frac{1+c_{diff}}{1-c_{diff}} \sigma_1$.
%Thus, we have $\alpha_t^2+\alpha_{t,\perp}^2+\beta_t^2+\beta_{t,\perp}^2 \le C_1 \sigma_1$.
%\end{proof}
%
%
%Now we prove Lemma~\ref{lem:rank_1_stage_1}.

\begin{proof}[Proof of Lemma~\ref{lem:rank_1_stage_1}]
	We use induction to prove the following statements for $t=0, 1, \ldots, T_1-1$:
	\begin{align*}
		\cD(t):\qquad & \alpha_t, \beta_t>0, \\
		\cE(t):\qquad & \xi_t \le \xi_0 \le 100 c_{init}^2\sigma_1, \\
		\cF(t):\qquad & \left(1+\frac{c_{step}}{3}\right) (\alpha_t+\beta_t)\le {\alpha_{t+1}+\beta_{t+1}} \le \left(1+c_{step}\right) (\alpha_t+\beta_t), \\
		\cG(t): \qquad & \abs{\alpha_t - \beta_t} \le \frac{99}{101}(\alpha_t+\beta_t), \\
		\cH(t): \qquad & \norm{\vect{u}_t}^2 + \norm{\vect{v}_t}^2 \le \sigma_1.
	\end{align*}
	
	\begin{itemize}
		\item Base cases.
		
		We know that $\cD(0)$, $\cE(0)$ and $\cG(0)$ hold from our assumptions on the initialization.
		
		\item $\cD(t), \cE(t) \Longrightarrow \cF(t)$ ($\forall t\le T_1-1$).
		
		From \eqref{eqn:rank-1-alpha-beta-dynamics} we have
		\begin{align*}
		{\alpha_{t+1}+\beta_{t+1}}  &={\left(1+\eta\sigma_1\right)\left(\alpha_t+\beta_t\right) - \eta\left(\alpha_t^2+\alpha_{t,\perp}^2\right)\beta_t - \eta \left(\beta_t^2+\beta_{t,\perp}^2\right)\alpha_t} \\
		&\ge  \left(1+\eta \sigma_1 - \eta \left(\alpha_t^2+\beta_t^2+\xi_t \right) \right)({\alpha_t+\beta_t}) \\
		&\ge  \left(1+\eta\sigma_1 - \eta \left( \frac{\sigma_1}{2} + 100 c_{init}^2\sigma_1 \right) \right)(\alpha_t+\beta_t) \\
		&\ge  \left(1+\frac{\eta \sigma_1}{3} \right)({\alpha_t+\beta_t}) \\
		&=  \left(1+\frac{c_{step}}{3} \right)({\alpha_t+\beta_t}),
		\end{align*}
		where in the second inequality we have used the definition of $T_1$, and the last inequality is true when $c_{init}$ is sufficiently small.
		
		On the other hand we have
		\begin{align*}
		{\alpha_{t+1}+\beta_{t+1}}  &= \left(1+\eta\sigma_1\right)\left(\alpha_t+\beta_t\right) - \eta\left(\alpha_t^2+\alpha_{t,\perp}^2\right)\beta_t - \eta \left(\beta_t^2+\beta_{t,\perp}^2\right)\alpha_t \\
		&\le \left(1+\eta\sigma_1\right)\left(\alpha_t+\beta_t\right) \\
		&= \left(1+c_{step} \right)\left(\alpha_t+\beta_t\right).
		\end{align*}
		
		\item $\cE(t) \Longrightarrow \cH(t)$ ($\forall t\le T_1-1$).
		
		We have
		\begin{align*}
		\norm{\vect{u}_t}^2 + \norm{\vect{v}_t}^2 = \alpha_t^2 + \beta_t^2 + \xi_t
		\le \frac12\sigma_1 + 100c_{init}^2 \sigma_1
		\le \sigma_1.
		\end{align*}
		
		\item $\cD(t), \cH(t) \Longrightarrow \cD(t+1)$ ($\forall t\le T_1-1$).
		
		From \eqref{eqn:rank-1-alpha-beta-dynamics} we have
		\begin{align*}
		\alpha_{t+1} = 
		\left(1-\eta\norm{\vect{v}_t}^2\right)\alpha_t + \eta\sigma_1\beta_t
		\ge \left(1-\eta\sigma_1 \right)\alpha_t
		= \left(1-c_{step} \right)\alpha_t >0.
		\end{align*}
		Similarly we have $\beta_{t+1}>0$. Note that $c_{step}$ is chosen to be sufficiently small.
		
		\item $\cH(t) \Longrightarrow \cE(t+1)$ ($\forall t\le T_1-1$).
		
		Recall from \eqref{eqn:rank-1-h-xi-dynamics}:
		\begin{align*}
		\xi_{t+1} = \left(1-\eta\norm{\vect{v}_t}^2\right)^2 \alpha_{t, \perp}^2 + \left(1-\eta\norm{\vect{u}_t}^2\right)^2 \beta_{t, \perp}^2.
		\end{align*}
		Since $\eta\norm{\vect{v}_t}^2 \le \eta(\norm{\vect{u}_t}^2+\norm{\vect{v}_t}^2) \le \eta\sigma_1 = c_{step} \le 1$ and $\eta\norm{\vect{u}_t}^2 \le 1$, we have
		\[
		\xi_{t+1} \le \alpha_{t, \perp}^2+\beta_{t, \perp}^2 = \xi_t.
		\]
		
		\item $\cD(t), \cE(t), \cF(t), \cG(t) \Longrightarrow \cG(t+1)$ ($\forall t\le T_1-1$).
		
		From \eqref{eqn:rank-1-alpha-beta-dynamics} we have
		\begin{align*}
		\alpha_{t+1}-\beta_{t+1} &= (1-\eta\sigma_1)(\alpha_t-\beta_t) - \eta (\beta_t^2 + \beta_{t, \perp}^2) \alpha_t + \eta (\alpha_t^2 + \alpha_{t, \perp}^2) \beta_t \\
		&= (1-\eta\sigma_1 + \eta \alpha_t\beta_t)(\alpha_t-\beta_t) - \eta \beta_{t, \perp}^2 \alpha_t + \eta  \alpha_{t, \perp}^2 \beta_t.
		\end{align*}
		From $\alpha_t^2 + \beta_t^2 < \frac12 \sigma_1$ we know $\alpha_t\beta_t < \frac14 \sigma_1$.
		Thus
		\begin{align*}
		\abs{\alpha_{t+1}-\beta_{t+1}} &\le (1-\eta\sigma_1 + \eta \alpha_t\beta_t)\abs{\alpha_t-\beta_t} + \eta \beta_{t, \perp}^2 \alpha_t + \eta  \alpha_{t, \perp}^2 \beta_t \\
		&\le \left( 1-\frac34 \eta\sigma_1 \right) \abs{\alpha_t-\beta_t} + \eta \xi_t (\alpha_t+\beta_t) \\
		&\le \left( 1-\frac34 \eta\sigma_1 \right) \cdot \frac{99}{101} (\alpha_t+\beta_t) + \eta \cdot 100c_{init}^2\sigma_1 (\alpha_t+\beta_t) \\
		&\le \left( 1- \eta\sigma_1 \left( \frac34 - 100c_{init}^2 \cdot \frac{101}{99} \right)  \right) \cdot \frac{99}{101} (\alpha_t+\beta_t) \\
		&\le \frac{99}{101} (\alpha_t+\beta_t) \\
		&\le \frac{99}{101} (\alpha_{t+1}+\beta_{t+1}).
		\end{align*}
	\end{itemize}

Lastly we upper bound $T_1$.
Note that for all $t<T_1$ we have $\alpha_t + \beta_t \le \sqrt{2(\alpha_t^2+\beta_t^2)} < \sqrt{2\cdot \frac12 \sigma_1} = \sqrt{\sigma_1}$.
%Also, from $\cG(t)$ we know $\frac{\alpha_t}{\beta_t} \in [\frac{1}{100}, 100]$.
%Hence we must have $\alpha_t + \beta_t = O(\sqrt{\sigma_1})$.
From $\cF(t)$ we know that $\alpha_t + \beta_t$ is increasing exponentially.
Therefore, we must have $T_1 = O\left( \log \frac{\sqrt{\sigma_1}}{\alpha_0 + \beta_0} \right) = O\left( \log \frac{\sqrt{\sigma_1}}{\sqrt{\sigma_1/d}} \right) = O(\log d)$.
%By the growth in signal space property, we have after $T_1 = O\left(\frac{1}{\eta \sigma_1}\log\left(\frac{\sqrt{\sigma_1}}{\alpha_0+\beta_0}\right)\right)$ iterations, $\alpha_{T_1}+\beta_{T_1} \ge \frac{\sqrt{\sigma_1}}{3}$.
%Note $T_1 \le T$ because $\alpha_{T_1}^2+\beta_{T_1}^2 \le (\alpha_{T_1}+\beta_{T_1})^2 = \frac{\sigma_1}{9}$.
%Combining with the bounded ratio between layers property we know the product $\alpha_{T_1}\beta_{T_1}$ is lower bounded by $c\sigma_1$.
\end{proof}

\begin{proof}[Proof of Lemma~\ref{lem:rank_1_stage_2}]
	
	By the definition of $T_1$ we know $\alpha_{T_1}^2 + \beta_{T_1}^2 \ge \frac12\sigma_1$.
	In the proof of Lemma~\ref{lem:rank_1_stage_1}, we have shown $\alpha_{T_1}, \beta_{T_1}>0$ and $\abs{\alpha_{T_1} - \beta_{T_1}} \le \frac{99}{101}\left( \alpha_{T_1}+\beta_{T_1} \right)$.
	These imply $\min\left\{\alpha_{T_1}, \beta_{T_1}\right\} \ge 2 \sqrt{c_1 \sigma_1}$ for some small universal constant $c_1>0$.
	
	We use induction to prove the following statements for all $t\ge T_1$:
	\begin{align*}
	\cI(t):\qquad & \alpha_t \ge \alpha_{T_1}\cdot \prod_{i=T_1}^{t-1} \left( 1- \eta\xi_0 \left( 1- c_1 c_{step}  \right)^{i-T_1} \right),\quad \beta_t \ge \beta_{T_1}\cdot \prod_{i=T_1}^{t-1} \left( 1- \eta\xi_0 \left( 1- c_1 c_{step}  \right)^{i-T_1} \right), \\
	\cJ(t):\qquad & \alpha_t, \beta_t \ge \sqrt{c_1\sigma_1}, \\
	\cK(t): \qquad & \alpha_t \beta_t \le \sigma_1, \text{ i.e.}, h_t\le 0, \\
	\cL(t):\qquad & \xi_{t} \le (1- c_1 c_{step})^{t-T_1}\xi_{0} \le (1- c_1 c_{step})^{t-T_1}\cdot 100c_{init}^2\sigma_1, \\
	\cM(t): \qquad & |h_{t+1}| \le (1-c_1 c_{step}) |h_t| + c_{step} \xi_t.
	\end{align*}
	
	\begin{itemize}
		\item Base cases.
		
		$\cI(T_1)$ is obvious. We know that $\cJ(T_1)$ is true by the definition of $c_1$.
		$\cK(T_1)$ can be shown as follows:
		\begin{align*}
		\alpha_{T_1} \beta_{T_1}
		&\le \frac14 \left( \alpha_{T_1}+\beta_{T_1} \right)^2 \\
		&\le \frac14  (1+c_{step})^2 \left( \alpha_{T_1-1}+\beta_{T_1-1} \right)^2 & \text{(by Lemma~\ref{lem:rank_1_stage_1} (iii))} \\
		&\le \frac14  (1+c_{step})^2 \cdot 2 \left( \alpha_{T_1-1}^2+\beta_{T_1-1}^2 \right) \\
		&\le \frac14  (1+c_{step})^2 \cdot 2 \cdot \frac12 \sigma_1 & \text{(by the definition of $T_1$)} \\
		&\le \sigma_1. & \text{(choosing $c_{step}$ to be small)}
		\end{align*}
		$\cL(T_1)$ reduces to $\xi_{T_1} \le \xi_0$, which was shown in the proof of Lemma~\ref{lem:rank_1_stage_1}.

		\item $\cI(t) \Longrightarrow \cJ(t)$ ($\forall t\ge T_1$).
		
		Notice that we have $\eta \xi_0 \le \frac{c_{step}}{\sigma_1} \cdot 100c_{init}^2 \sigma_1 = 100c_{step}c_{init}^2 < \frac12$ since $c_{step}$ and $c_{init}$ are sufficiently small.
		Then we have
		\begin{align*}
		\alpha_t &\ge \alpha_{T_1}\cdot \prod_{i=T_1}^{t-1} \left( 1- \eta\xi_0 \left( 1- c_1 c_{step}  \right)^{i-T_1} \right) \\
		&\ge \alpha_{T_1}\cdot \prod_{i=0}^\infty \left( 1- \eta\xi_0 \left( 1- c_1 c_{step}  \right)^{i} \right) \\
		&\ge \alpha_{T_1}\cdot \prod_{i=0}^\infty \exp\left( - 2\eta\xi_0 \left( 1- c_1 c_{step}  \right)^{i} \right) & \text{($1-x\ge e^{-2x}$, $\forall 0\le  x\le 1/2$)} \\
		&= \alpha_{T_1}\cdot \exp\left( -\frac{2\eta\xi_0}{c_1c_{step}} \right) \\
		&\ge \alpha_{T_1}\cdot \exp\left( -\frac{200c_{step}c_{init}^2}{c_1c_{step}} \right) \\
		&\ge 2\sqrt{c_1\sigma_1}\cdot \exp\left( -\frac{200c_{init}^2}{c_1} \right) \\
		&\ge \sqrt{c_1\sigma_1}. & \text{(choosing $c_{init}$ to be small)}
		\end{align*}
		Similarly we have $\beta_t \ge \sqrt{c_1\sigma_1}$.
		
		\item $\cI(t), \cJ(t), \cK(t), \cL(t) \Longrightarrow \cI(t+1)$ ($\forall t\ge T_1$).
		
		From \eqref{eqn:rank-1-alpha-beta-dynamics} we have
		\begin{align*}
		\alpha_{t+1} &= \left(1-\eta\left(\beta_t^2 +\beta_{t,\perp}^2\right)\right)\alpha_t + \eta\sigma_1\beta_t \\
		&= \left( 1 - \eta \beta_{t, \perp}^2 \right)\alpha_t - \eta h_t\beta_t \\
		&\ge \left( 1 - \eta \beta_{t, \perp}^2 \right)\alpha_t &\text{($h_t\le0, \beta_t>0$)} \\
		&\ge \left( 1 - \eta \xi_t \right)\alpha_t \\
		&\ge \left( 1 - \eta \xi_{0}(1- c_1 c_{step})^{t-T_1} \right)\alpha_t & (\cL(t)) \\
		&\ge \alpha_{T_1}\cdot \prod_{i=T_1}^{t} \left( 1- \eta\xi_0 \left( 1- c_1 c_{step}  \right)^{i-T_1} \right). & (\cI(t))
		\end{align*}
		Similarly we have $\beta_{t+1} \ge \beta_{T_1}\cdot \prod_{i=T_1}^{t} \left( 1- \eta\xi_0 \left( 1- c_1 c_{step}  \right)^{i-T_1} \right)$.

		\item $\cJ(t), \cK(t), \cL(t) \Longrightarrow \cK(t+1)$ ($\forall t\ge T_1$).
		
		From \eqref{eqn:rank-1-h-xi-dynamics} we have
		\begin{equation} \label{eqn:rank-1-inproof-1}
		\begin{aligned}
		h_{t+1}  &= \left(1-\eta\left(\alpha_t^2+\beta_t^2\right)
		+\eta^2\left(\alpha_t \beta_t h_t+\alpha_t^2\beta_{t,\perp}^2+\beta_t^2\alpha_{t,\perp}^2 + \alpha_{t,\perp}^2\beta_{t,\perp}^2\right)
		\right)h_t 
		-\eta \alpha_t \beta_t \xi_t  + \eta^2\sigma_1\alpha_{t,\perp}^2\beta_{t,\perp}^2 \\
		&\le \left(1-\eta\left(\alpha_t^2+\beta_t^2\right) \right) h_t + \eta^2 \alpha_t\beta_t h_t^2 -\eta \alpha_t \beta_t \xi_t  + \eta^2\sigma_1\alpha_{t,\perp}^2\beta_{t,\perp}^2,
		\end{aligned}
		\end{equation}
		where we have used $h_t\le 0$.
		Since $\alpha_t, \beta_t \ge \sqrt{c_1\sigma_1}$ and $\alpha_t\beta_t \le \sigma_1$, we have $\alpha_t, \beta_t = \Theta(\sqrt{\sigma_1})$.
		Furthermore, we can choose $c_{step}$ and $c_{init}$ small enough such that $\eta \xi_0 \le 4c_1$ which implies
		\begin{align*}
		\eta^2\sigma_1\alpha_{t,\perp}^2\beta_{t,\perp}^2
		\le \eta^2 \sigma_1 \cdot \frac14 \xi_t^2
		\le \frac14 \eta\xi_t \cdot \eta \sigma_1    \xi_0 
		\le  \eta\xi_t \cdot c_1 \sigma_1   
		\le  \eta\xi_t \cdot \alpha_t \beta_t.
		\end{align*}
		Therefore \eqref{eqn:rank-1-inproof-1} implies
		\begin{align*}
		h_{t+1} &\le \left( 1 - \eta \cdot O(\sigma_1) \right) h_t + \eta^2 \alpha_t\beta_t h_t^2 \\
		&= \left( 1 - \eta \cdot O(\sigma_1) + \eta^2 \alpha_t\beta_t h_t \right) h_t \\
		&\le \left( 1 - \eta \cdot O(\sigma_1) - \eta^2 \sigma_1^2 \right) h_t & (0<\alpha_t\beta_t\le \sigma_1)\\
		&= \left( 1 -   O(c_{step}) - c_{step}^2 \right) h_t \\
		&\le0,
		\end{align*}
		where the last step is true when $c_{step}$ is sufficiently small.
		
		\item $\cJ(t), \cK(t), \cL(t) \Longrightarrow \cL(t+1)$ ($\forall t\ge T_1$).
		
		From $\alpha_t, \beta_t \ge \sqrt{c_1\sigma_1}$ and $\alpha_t\beta_t \le \sigma_1$ we have $\alpha_t, \beta_t = \Theta(\sqrt{\sigma_1})$.
		Also we have $\xi_t \le \xi_0$.
		Thus we can make sure $\eta (\alpha_t^2 + \alpha_{t, \perp}^2) < 1$ and $\eta (\beta_t^2 + \beta_{t, \perp}^2) < 1$. Then from \eqref{eqn:rank-1-h-xi-dynamics} we have
		\begin{align*}
		\xi_{t+1}   &= \left(1-\eta\left(\beta_t^2+\beta_{t,\perp}^2\right)\right)^2 \alpha_{t,\perp}^2 + \left(1-\eta\left(\alpha_t^2+\alpha_{t,\perp}^2\right)\right)^2 \beta_{t,\perp}^2 \\
		&\le  \left(1-\eta \beta_t^2\right)^2\alpha_{t,\perp}^2 + \left(1-\eta \alpha_t^2\right)\beta_{t,\perp}^2\\
		&\le \left(1-\eta c_1\sigma_1\right)\xi_t\\
		&=  \left(1-c_1c_{step} \right)\xi_t.
		\end{align*}
		
		\item We have shown $\cI(t), \cJ(t), \cK(t)$ and $\cL(t)$ for all $t\ge T_1$. Now we use them to prove $\cM(t)$ for all $t\ge T_1$:
		
		\begin{align*}
		|h_{t+1}| &= \left(1-\eta\left(\alpha_t^2+\beta_t^2\right)
		+\eta^2\left(\alpha_t \beta_t h_t+\alpha_t^2\beta_{t,\perp}^2+\beta_t^2\alpha_{t,\perp}^2 + \alpha_{t,\perp}^2\beta_{t,\perp}^2\right)
		\right) |h_t| 
		+ \eta \alpha_t \beta_t \xi_t  - \eta^2\sigma_1\alpha_{t,\perp}^2\beta_{t,\perp}^2 \\
		&\le \left(1- \frac12\eta\left(\alpha_t^2+\beta_t^2\right) \right) |h_t| + \eta \alpha_t \beta_t \xi_t  \\
		&\le \left(1- \frac12\eta\cdot 2 c_1\sigma_1 \right) |h_t| + \eta \sigma_1 \xi_t  \\
		&= \left(1- c_1c_{step} \right) |h_t| + c_{step} \xi_t.
		\end{align*}
		Here we have used $\eta \le \frac{\alpha_t^2+\beta_t^2}{2 \abs{ \alpha_t \beta_t h_t+\alpha_t^2\beta_{t,\perp}^2+\beta_t^2\alpha_{t,\perp}^2 + \alpha_{t,\perp}^2\beta_{t,\perp}^2 }}$, which is clearly true when $c_{step}$ is small enough.
	\end{itemize}

Therefore, we have finished the proof of Lemma~\ref{lem:rank_1_stage_2}.
\end{proof}

\end{document}